\documentclass{article}

\usepackage[preprint]{tmlr}

\usepackage[utf8]{inputenc} 
\usepackage{url}            
\usepackage{booktabs,tabularx,array} 
\usepackage{amsfonts}       
\usepackage{pifont}
\usepackage{nicefrac}       
\usepackage{microtype}      
\usepackage{xcolor}         
\usepackage{enumitem}
\usepackage{multirow}
\usepackage{graphicx}
\usepackage{wrapfig}
\usepackage{caption}
\usepackage{algorithm}
\usepackage{algpseudocode}
\usepackage{amsmath, amssymb}
\usepackage{mathtools}
\usepackage{amsthm}
\usepackage{tikz}
\usepackage{colortbl}

\newcolumntype{L}[1]{>{\raggedright\arraybackslash}p{#1}}

\definecolor{checkmark}{HTML}{008000}
\definecolor{cross}{HTML}{FF0000}

\usepackage[normalem]{ulem}
\definecolor{citeblue}{RGB}{0,90,180}
\usepackage[
    colorlinks=true,
    citecolor=citeblue
]{hyperref}

\usepackage{soul}
\colorlet{hlblue}{citeblue!12}
\newcommand{\hlbox}[1]{{\sethlcolor{hlblue}\hl{#1}}}

\usepackage[capitalize,noabbrev]{cleveref}

\makeatletter
\providecommand*{\theHALG@line}{\thealgorithm.\arabic{ALG@line}}
\makeatother

\theoremstyle{plain}
\newtheorem{theorem}{Theorem}[section]
\newtheorem{proposition}[theorem]{Proposition}
\newtheorem{lemma}[theorem]{Lemma}

\theoremstyle{definition}

\newtheorem{remark}[theorem]{Remark}

\title{Spectrum Matching: a Unified Perspective for Superior \mbox{Diffusability} in Latent Diffusion}

\author{\name Mang Ning \email m.ning@uu.nl \\
      \addr Utrecht University
      \AND
      \name Mingxiao Li \email mingxiao.li@kuleuven.be \\
      \addr KU Leuven
      \AND
      \name Le Zhang \email le.zhang@mila.quebec \\
      \addr Mila
      \AND
      \name Lanmiao Liu \email lanmiao.liu@mpi.nl \\
      \addr Utrecht University \\
      Max Planck Institute for Psycholinguistics
      \AND
      \name Matthew B.\ Blaschko \email matthew.blaschko@kuleuven.be \\
      \addr KU Leuven
      \AND
      \name Albert Ali Salah \email a.a.salah@uu.nl \\
      \addr Utrecht University
      \AND
      \name Itir Onal Ertugrul \email i.onalertugrul@uu.nl \\
      \addr Utrecht University}

\begin{document}

\maketitle

\begin{abstract}
In this paper, we study the diffusability (learnability) of variational autoencoders (VAE) in latent diffusion. We first derive distribution-free bounds on the maximum achievable reduction in pixel-space denoising loss. Under a power-law power spectral density (PSD), these bounds establish an ordering between sufficiently separated frequencies, with larger achievable reductions at lower frequencies. Motivated by this analysis and measurements of a trained pixel diffusion model, we seek to make low-frequency denoising gains useful for generating perceptually meaningful semantic content in latent diffusion. We propose the \textbf{Spectrum Matching Hypothesis}: latents with superior diffusability should (i) follow a flattened power-law PSD (\textbf{Encoding Spectrum Matching}, ESM) and (ii) preserve frequency-to-frequency semantic correspondence through the decoder (\textbf{Decoding Spectrum Matching}, DSM). In practice, we apply ESM by matching the PSD between images and latents, and DSM via shared spectral masking with frequency-aligned reconstruction. Importantly, Spectrum Matching clarifies prior observations of over-noisy or over-smoothed latents and provides a unified spectral interpretation of prior methods, including VA-VAE and EQ-VAE. Experiments suggest that Spectrum Matching yields superior diffusion generation on CelebA and ImageNet datasets, and outperforms prior approaches. Finally, we extend the spectral view to representation alignment (REPA): we show that the directional spectral energy of the target representation is crucial for REPA, and propose a DoG-based method to further improve the performance of REPA. Our code is available at \url{https://github.com/forever208/SpectrumMatching}.
  
\end{abstract}

\section{Introduction}
\label{sec: Introduction}

Latent diffusion models have become a main paradigm for high-resolution image generation \citep{rombach2022high} and video generation \citep{wan2025wan,wu2025hunyuanvideo,ma2025step}, combining the expressive power of diffusion models \citep{ho2020denoising,songscore,peebles2023scalable,liuflow} with the computational efficiency of operating in a compressed latent space. In this two-stage framework, a first-stage Variational Autoencoder (VAE) maps images to latents, and a second-stage diffusion model learns to generate these latents, which are then decoded back to RGB space. This design underpins many modern text-to-image and unconditional generation systems \citep{labs2025flux,wu2025qwen,cai2025z}, enabling high-resolution synthesis with manageable training and inference cost.

Despite their success, latent diffusion models exhibit a practically important problem: \textbf{better reconstructions do not necessarily imply better generation quality}. Recent studies show that reconstruction-focused improvements to the VAE can yield limited or even inconsistent gains in downstream diffusion quality \citep{hansenlearnings}, motivating a shift from reconstruction fidelity to the diffusability (learnability) of the latent representation \citep{skorokhodov2025improving}. This perspective has inspired a growing body of work that regularizes the latent space to make it easier for diffusion to model. For example, prior methods suggest that non-uniform (biased) latent spectra can be beneficial \citep{liu2025delving}, aligning latents to pretrained foundation-model features improves diffusion performance \citep{yao2025reconstruction,fan2025prism}, and truncating high-frequency latent components via downsampling or enforcing equivariance to spatial transforms can improve generation \citep{skorokhodov2025improving,kouzelis2025eq}. While these findings are compelling, they are often presented as separate observations or heuristics, leaving open a central question: What properties characterize a diffusion-friendly latent space?

In this work, we propose a unifying answer through the lens of the latent spectrum. We derive distribution-free bounds on the maximum achievable reduction in pixel-space denoising Mean Squared Error (MSE). Under a power-law power spectral density (PSD), lower frequencies admit larger achievable reductions than sufficiently higher frequencies at a given noise level. We also observe that achieved loss reduction is concentrated at low frequencies in a trained pixel diffusion model. Low-frequency image components capture coarse shapes and spatial layouts that contribute to visual semantics, suggesting that this denoising advantage may support the generation of perceptually meaningful semantic content. However, VAE encoding can redistribute image information across latent frequency bands, so low-frequency latent components need not correspond to low-frequency image content \citep{lai2025toward}. To extend this spectral advantage to latent diffusion, we introduce the \textbf{Spectrum Matching Hypothesis}: latents with superior diffusability should (i) follow a flattened power-law PSD (\textbf{Encoding Spectrum Matching}, ESM), and (ii) preserve frequency-to-frequency semantic correspondence through the decoder (\textbf{Decoding Spectrum Matching}, DSM). This hypothesis not only naturally yields practical algorithms, but also provides a unified spectral interpretation of prior observations and methods, including over-noisy latents, VA-VAE and EQ-VAE.

Beyond VAE latents, we further show that the spectrum view can clarify representation alignment (REPA) \citep{yu2025repa}, a recently successful paradigm for accelerating diffusion training with feature-based alignment. We demonstrate that the proposed RMS Spatial Contrast (RMSC) metric in iREPA \citep{singh2025matters} is equivalent to \textbf{directional spectral energy}, suggesting that the spectral energy of the direction field is a key property of effective target representations. Moreover, we propose a Difference-of-Gaussians (DoG) band-pass preprocessing that improves REPA generation quality.

To summarize, our contributions are fourfold:
\begin{itemize}
    \item We characterize the frequency dependence of achievable denoising MSE reduction under a power-law PSD, and empirically observe that achieved loss reduction is concentrated at low frequencies in a trained pixel diffusion model.
    
    \item We propose the Spectrum Matching Hypothesis for latent diffusion, which provides a unified spectral perspective on prior methods and empirical observations.
    
    \item We instantiate ESM via PSD matching, implement DSM via shared spectral masking and frequency-aligned reconstruction, leading to superior latent diffusability.
    
    \item We extend the spectrum view to REPA by connecting RMSC to directional spectral energy, and introduce a DoG-based method that improves REPA and iREPA.
\end{itemize}

\section{Related Work}
\label{sec: Related Work}

\paragraph{VAE in Latent Diffusion.}
The two-stage latent diffusion models (LDM) were introduced in \citet{rombach2022high} for high-resolution image generation, and the VAE used for the first-stage compression has been widely studied for better reconstruction or generation. On the reconstruction side, the SDXL approach \citep{podellsdxl} showed that larger batch sizes and exponential moving average (EMA) updates improve reconstruction quality. SD3-VAE \citep{esser2024scaling} and Flux-VAE \citep{labs2025flux} further boosted reconstruction quality by increasing latent channel capacity. To achieve higher compression ratios, DC-AE \citep{chen2025deep} introduced a residual module together with a multi-phase training strategy. Other lines of work explicitly decoupled the reconstruction of low and high-frequency components to better reconstruct the fine details \citep{medi2025missing,fan2025prism}. Beyond reconstruction, several methods aim to improve downstream diffusion performance by regularizing the VAE. A common strategy is to inject perturbations into the latent space during VAE training \citep{chen2025dc,team2025nextstep,zheng2026diffusion}, which helps by mitigating exposure bias in diffusion models \citep{ning2023input,ning2023elucidating}. More recently, researchers have found that a lossy or weak encoder is also feasible for diffusion modeling by enhancing the capability of the decoder \citep{zheng2026diffusion,zhaoepsilon}.

\paragraph{Diffusability of the Latent Representations.}
\label{subsec: Diffusability of the Latent Representations}
A VAE with strong reconstruction fidelity does not necessarily yield better downstream diffusion performance \citep{hansenlearnings}. This empirical observation has motivated recent work to study the \emph{diffusability} of the latent space. For instance, \citet{liu2025delving} argues that latents with a \emph{biased} (non-uniform) spectrum are preferable for diffusion, highlighting the importance of latent spectral structure. Another line of work improves diffusability by aligning VAE latents with representations from foundation models. VA-VAE \citep{yao2025reconstruction} and UAE \citep{fan2025prism} reveal that matching latents to features such as DINOv2 \citep{oquab2024dinov2} can substantially enhance diffusion quality. As we discuss in Section~\ref{subsec: Spectrum Matching Unifies Prior Observations and Approaches}, these feature-alignment approaches can be interpreted through the lens of Spectrum Matching, where the pretrained representation implicitly defines a desirable target spectrum. In addition, Scale Equivariance \citep{skorokhodov2025improving} reports that standard VAEs often exhibit an abnormally strong high-frequency component in the latent space, and proposes to truncate these frequencies via latent downsampling. EQ-VAE \citep{kouzelis2025eq} further enforces equivariance of latents under spatial transformations, which also improves diffusion performance. In Section~\ref{subsec: Spectrum Matching Unifies Prior Observations and Approaches}, we discuss spectral connections between these methods and Spectrum Matching through encoder-side spectral structure and decoder-side frequency consistency.

\paragraph{Differences from Scale Equivariance.}
Scale Equivariance \cite{skorokhodov2025improving} primarily focuses on suppressing excessive high-frequency artifacts in latent space through latent downsampling or frequency truncation. In contrast, our work studies a broader question: what spectral properties and semantic organization should a diffusion-friendly latent representation possess? Besides, Scale Equivariance hypothesizes that latent representations should exhibit the same spectral distribution as RGB images. In contrast, motivated by our theoretical analysis, we hypothesize that diffusion-friendly latents should retain a power-law PSD that is flatter than that of natural images, which motivates our ESM principle. Furthermore, beyond the spectral shape itself, we highlight the importance of preserving frequency-to-frequency correspondence between latent and image space (not discussed in \cite{skorokhodov2025improving}), leading to our DSM principle.

\section{Spectrum Matching}
\label{sec: Spectrum Matching}

In this section, we first introduce Proposition~\ref{proposition:1}, which characterizes a low-frequency advantage in achievable denoising MSE reduction under the power-law PSD of natural images. Together with empirical evidence from a trained pixel diffusion model, this finding suggests a low-frequency bias in denoising gains. Since low-frequency components capture coarse shapes and spatial layouts that contribute to visual semantics, this bias is beneficial for the generation of perceptually meaningful semantic content. To extend this spectral advantage to latent diffusion, we propose the Spectrum Matching Hypothesis for the latent space and instantiate it through ESM and DSM.

\subsection{Power-Law PSD and Denoising MSE Reduction}
\label{subsec: Power-Law PSD and Denoising MSE Reduction}

\begin{proposition}[\textbf{Power-law PSD and achievable denoising MSE reduction}]

Let $\pmb{x}_0$ be a random natural image with finite second moments, and let $y_0(\omega)\triangleq \mathcal{F}(\pmb{x}_0)(\omega)$ be its Fourier coefficients under the unitary DFT $\mathcal{F}$, with power spectral density
$S(\omega)\triangleq \mathbb{E}\!\left[|y_0(\omega)|^2\right]$.
The diffusion forward process $\pmb{x}_t=\sqrt{\bar\alpha_t}\,\pmb{x}_0+\sqrt{1-\bar\alpha_t}\,\pmb{\varepsilon}$ (where $\pmb{\varepsilon}\sim\mathcal{N}(\pmb{0},\pmb{I})$) implies the diffusion in the Fourier domain:
\[
y_t(\omega)=\sqrt{\bar\alpha_t}\,y_0(\omega)+\sqrt{1-\bar\alpha_t}\,\eta(\omega),
\]
where $\eta=\mathcal{F}(\pmb{\varepsilon})$ is spectrally flat, i.e., $\mathbb{E}[\eta(\omega)\eta(\omega')^*]=\mathbf{1}[\omega=\omega']$, and independent of $\pmb{x}_0$. Define the per-frequency timestep signal-to-noise ratio (SNR) as $\mathrm{SNR}_t(\omega)\triangleq \frac{\bar\alpha_t\,S(\omega)}{1-\bar\alpha_t}$, and let $\Delta_t(\omega)$ denote the \emph{maximum achievable reduction in per-frequency MSE relative to zero prediction}, over all denoisers $f$ observing the \emph{full} noisy image $\pmb{x}_t$:
\[
\Delta_t(\omega)\;\triangleq\; S(\omega)-\inf_{f}\,\mathbb{E}\Big[\big|y_0(\omega)-\mathcal{F}\big(f(\pmb{x}_t)\big)(\omega)\big|^2\Big].
\]
Then, without any distributional assumption on $\pmb{x}_0$,

\noindent
\begin{equation}
\label{eq: proposition1}
G_t(\omega)\;\le\;\Delta_t(\omega)\;\le\;S(\omega),
\qquad
G_t(\omega)\;\triangleq\; S(\omega)\cdot\frac{\mathrm{SNR}_t(\omega)}{1+\mathrm{SNR}_t(\omega)},
\end{equation}

\noindent
where $G_t(\omega)$ is attained by the optimal same-frequency linear (Wiener) estimator $c_t(\omega)\,y_t(\omega)$ with $c_t(\omega)=\frac{\sqrt{\bar\alpha_t}\,S(\omega)}{\bar\alpha_t S(\omega)+1-\bar\alpha_t}$.
Consequently, for natural images with $S(\omega)=K|\omega|^{-\alpha}$ ($K,\alpha>0$, $\omega\neq 0$),

\noindent
\begin{equation}
\label{eq: proposition1 ordering}
\Delta_t(\omega_1)\;\ge\;\Delta_t(\omega_2)
\quad\text{whenever}\quad
|\omega_2|\;\ge\;\Big(1+\tfrac{1}{\mathrm{SNR}_t(\omega_1)}\Big)^{1/\alpha}|\omega_1|,
\end{equation}

\noindent
Thus, at each noise level, a lower frequency admits at least as large a maximum reduction in MSE as any sufficiently higher frequency, with the required separation specified by Eq.~\eqref{eq: proposition1 ordering}. This separation decreases as $\mathrm{SNR}_t(\omega_1)$ increases (proof in Appendix~\ref{append: Proof of Proposition}).
\label{proposition:1}
\end{proposition}

\begin{wrapfigure}{r}{0.4\textwidth}
\vskip -0.15in
  \includegraphics[width=0.4\textwidth]{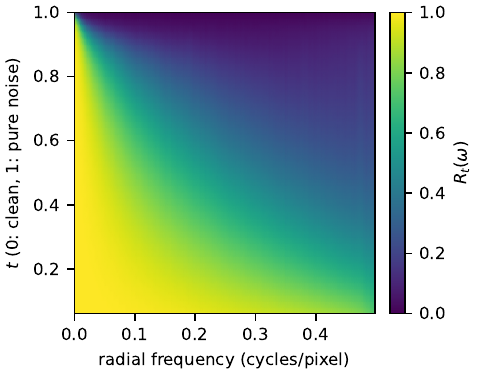}
  \captionsetup{skip=1pt}
  \caption{Class-averaged predictability $R_t(\omega)$ of a pixel diffusion model (JiT-H/16 on ImageNet $256\times256$).}
\label{fig: predictability_heatmap}
\vskip -0.1in
\end{wrapfigure}

\paragraph {Interpretation of Proposition \ref{proposition:1}.} Under the unitary DFT, pixel-space MSE is a sum of equally weighted per-frequency errors. Proposition~\ref{proposition:1} characterizes how much of this error can be removed at each frequency. Under a power-law PSD, low frequencies have both higher signal power and higher SNR. Their guaranteed MSE reduction $G_t(\omega)$ can therefore exceed the entire signal power of sufficiently higher frequencies, yielding the ordering in Eq.~\eqref{eq: proposition1 ordering}. This result concerns optimal denoising risk at a given noise level.

We observe a related concentration of achieved loss reduction in a trained JiT model \citep{li2026back} (Figure~\ref{fig: predictability_heatmap}). Low frequencies remain predictable at higher noise levels than high frequencies. Under the evaluation protocol in Appendix~\ref{append: predictability}, frequencies below $0.1$ cycles/pixel constitute only $3.2\%$ of the Fourier coefficients but account for approximately $91\%$ of the within-class spectral variance and $95\%$ of the reported positive MSE reduction relative to the class-mean predictor. Frequencies above $0.25$ cycles/pixel constitute $80\%$ of the coefficients but account for only $1\%$ of this reduction. These measurements show that \textbf{the denoising improvement achieved by the evaluated model is concentrated at low spatial frequencies}. These findings complement existing analyses of diffusion training dynamics and score smoothness \citep{bonnaire2025diffusion}, as well as observations that improved high-frequency modeling does not necessarily lead to better perceptual quality \citep{falck2025fourier}.

\subsection{Spectrum Matching Hypothesis}
\label{subsec: Spectrum Matching Hypothesis}

Motivated by the low-frequency denoising advantage and the measurements in Section~\ref{subsec: Power-Law PSD and Denoising MSE Reduction}, we aim to organize latent representations so that this advantage can support the generation of perceptually meaningful semantic content. This motivates jointly shaping latent spectral energy and preserving its correspondence to image content through the decoder. We formulate this as the \emph{Spectrum Matching Hypothesis}. Given a VAE consisting of an encoder $E(\cdot)$ and a decoder $D(\cdot)$, we hypothesize that a latent representation $\pmb{z} = E(\pmb{x})$ with improved diffusability has the following properties:

(i) \textbf{Encoding Spectrum Matching (ESM).} \hlbox{We hypothesize that latents retain the power-law structure of natural images with a flatter PSD, $S_{\pmb{z}}(\omega)\propto |\omega|^{-(\alpha-\delta)}$.} Here, $\alpha$ is the image spectral exponent, and $\delta>0$ is chosen empirically. The flattening is motivated by the maximum-entropy reference model in Appendix~\ref{append: ESM from an Information Theory Perspective}.

(ii) \hlbox{\textbf{Decoding Spectrum Matching (DSM)}, where the decoder $D(\cdot)$ should be frequency-aligned such that latent frequency bands can be decoded to corresponding image frequency bands}. For example, the low-frequency bands of the latent $\pmb{z}$ should contain the low-frequency information of the input image $\pmb{x}$.

Essentially,  ESM regularizes the spectral energy and DSM aims to preserve frequency-to-frequency semantic correspondence between latent and image space. Also, Spectrum Matching preserves the coarse-to-fine (spectral autoregressive \cite{dieleman2024spectral}) generation order: latent diffusion can first generate low-frequency structure and progressively refine higher-frequency details of the RGB image.

\subsection{Algorithms of ESM and DSM}
\label{subsec: Practical Implementation of ESM and DSM}

In order to apply the Spectrum Matching regularization in VAE, we propose practical algorithms for ESM and DSM, respectively. Figure~\ref{fig: pipeline} illustrates how these two methods are integrated into a standard VAE training used for latent diffusion training.

\begin{figure}[htb]
\vskip -0.1in
\begin{center}
\centerline{\includegraphics[width=0.85\columnwidth]{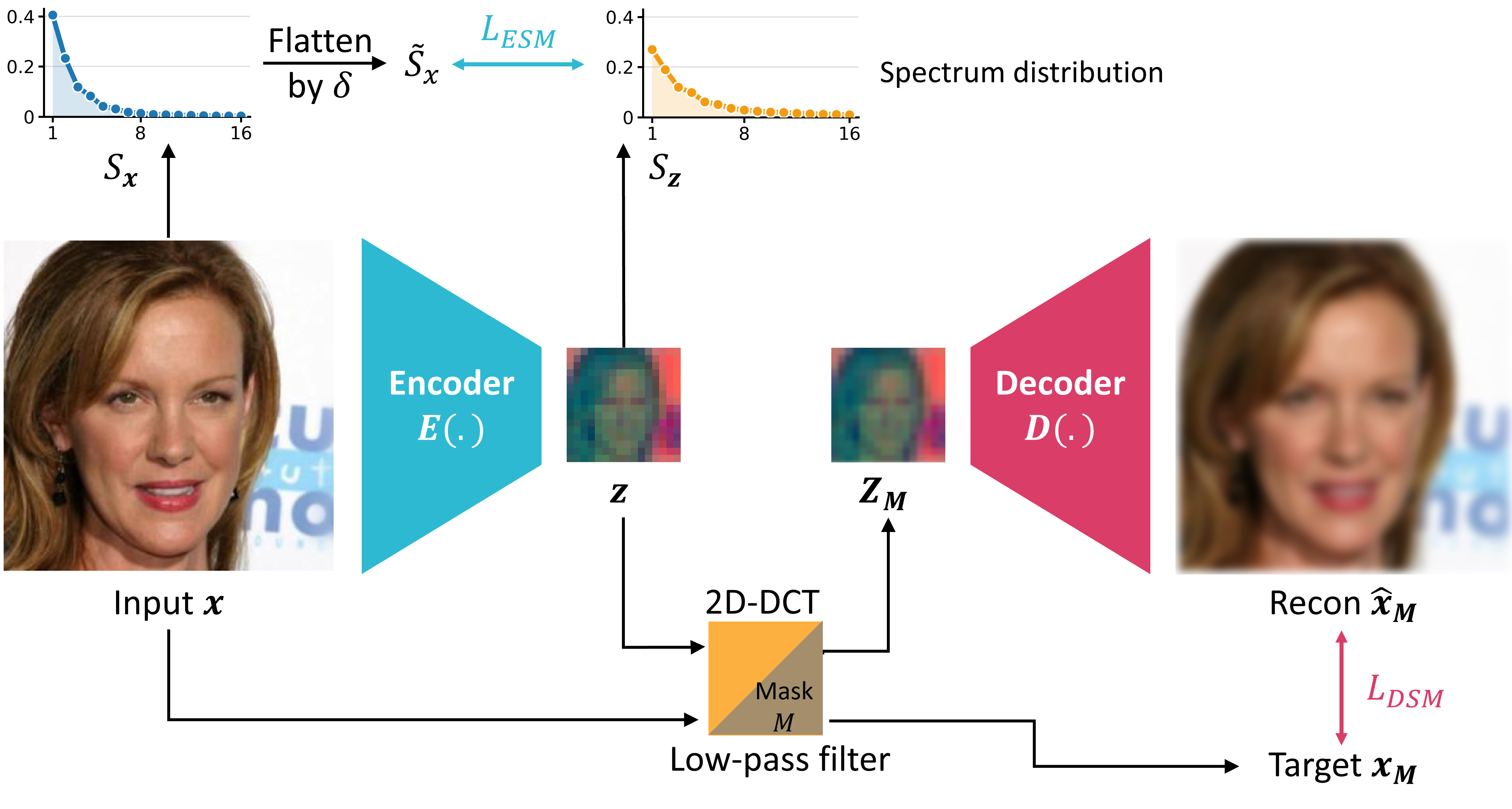}}
\caption{Diagram of ESM and DSM in a typical VAE for latent diffusion. 
}
\label{fig: pipeline}
\end{center}
\vskip -0.2in
\end{figure}

\paragraph{Encoding Spectrum Matching (ESM).} ESM regularizes the encoder-side latent spectrum to make it more learnable by diffusion training. As shown in Algorithm~\ref{alg:esm}, given an input image $\pmb{x}$, we first obtain its latent representation $\pmb{z}=E(\pmb{x})$. We then compute a spectral descriptor by PSD for both the image and the latent, denoted by $S_{\pmb{x}}$ and $S_{\pmb{z}}$, respectively. Following the ESM hypothesis, we construct a flattened image-side target spectrum $\tilde{S}_{\pmb{x}}=\mathrm{Flatten}(S_{\pmb{x}};\delta)$ with $\delta>0$. A larger $\delta$ produces a flatter target spectrum, while reconstruction losses help preserve image content. Finally, both spectra are normalized into valid distributions $\hat{S}_{\pmb{x}}$ and $\hat{S}_{\pmb{z}}$, and the ESM loss is defined as a KL divergence $\mathcal{L}_{\mathrm{ESM}}
= \mathrm{KL}\!\left(\hat{S}_{\pmb{x}} \,\|\, \hat{S}_{\pmb{z}}\right)$.

In practice, researchers often use the following compound loss to train a VAE:

\noindent
\begin{equation}
\label{eq: SDVAE loss}
\mathcal{L}_{\mathrm{SD-VAE}}(\pmb{x})
= \mathcal{L}_{1}(\pmb{x}, \hat{\pmb{x}}) + 
\lambda_1 \mathcal{L}_{\mathrm{LPIPS}}(\pmb{x}, \hat{\pmb{x}}) + \lambda_2 \mathcal{L}_{\mathrm{GAN}}(\pmb{x}, \hat{\pmb{x}}) + \lambda_3 \mathcal{L}_{\mathrm{KL}}.
\end{equation}
\noindent

In the case of ESM, we integrate the loss $\mathcal{L}_{\mathrm{ESM}}$ into the VAE losses by:

\noindent
\begin{equation}
\label{eq: VAE-ESM loss}
\mathcal{L}_{\mathrm{ESM-AE}}(\pmb{x})
= \mathcal{L}_{1}(\pmb{x}, \hat{\pmb{x}}) + 
\lambda_1 \mathcal{L}_{\mathrm{LPIPS}}(\pmb{x}, \hat{\pmb{x}}) + \lambda_2 \mathcal{L}_{\mathrm{GAN}}(\pmb{x}, \hat{\pmb{x}}) + \beta \mathcal{L}_{\mathrm{ESM}}.
\end{equation}
\noindent

where $\beta$ is a hyperparameter for the ESM loss. Following DC-AE \citep{chen2025deep} and RAE \citep{zheng2026diffusion}, we remove the KL loss term when training SD-VAE, ESM-AE or DSM-AE from scratch. Note that the computational cost of $\mathcal{L}_{\mathrm{ESM}}$ is negligible, so that ESM is an efficient regularization method.

\begingroup
\captionsetup{hypcap=false}
\begin{figure}[t]
\begin{minipage}[t]{0.485\textwidth}
\begin{algorithm}[H]
\caption{Encoding Spectrum Matching (ESM)}
\label{alg:esm}
\begin{algorithmic}
\Require Image $\pmb{x}$, encoder $E(\cdot)$, factor $\delta \ge 0$

\State 1: Compute image latent: $\pmb{z} \leftarrow E(\pmb{x})$

\State 2: Compute image and latent PSDs:
\State \hspace{2em} $S_{\pmb{x}} \leftarrow \mathrm{PSD}(\pmb{x})$, \quad  $S_{\pmb{z}} \leftarrow \mathrm{PSD}(\pmb{z}) $

\State 3: Flatten the image-side target spectrum:
\State \hspace{2em} $\tilde{S}_{\pmb{x}} \leftarrow \mathrm{Flatten}(S_{\pmb{x}};\delta)$

\State 4: Normalize spectra into distributions:
\State \hspace{2em} $\hat{S}_{\pmb{x}} \leftarrow \mathrm{Normalize}(\tilde{S}_{\pmb{x}})$, \quad
$\hat{S}_{\pmb{z}} \leftarrow \mathrm{Normalize}(S_{\pmb{z}})$

\State 5: Match latent PSD to target PSD:
\State \hspace{2em} $\mathcal{L}_{\mathrm{ESM}} \leftarrow \mathrm{KL \ divergence}(\hat{S}_{\pmb{x}}, \hat{S}_{\pmb{z}})$

\State \Return $\mathcal{L}_{\mathrm{ESM}}$
\end{algorithmic}
\end{algorithm}
\end{minipage}\hfill
\begin{minipage}[t]{0.485\textwidth}
\begin{algorithm}[H]
\caption{Decoding Spectrum Matching (DSM)}
\label{alg:dsm}
\begin{algorithmic}
\Require Image $\pmb{x}$, encoder $E(\cdot)$, decoder $D(\cdot)$, frequency mask family $\mathcal{M}$

\State 1: Compute image latent: $\pmb{z} \leftarrow E(\pmb{x})$

\State 2: Sample a shared low-pass mask: $M \sim \mathcal{M}$

\State 3: Apply the same mask to image and latent:
\State \hspace{2em} $\pmb{x}_{M} \leftarrow \mathrm{SpectralFilter}(\pmb{x}, M)$
\State \hspace{2em} $\pmb{z}_{M} \leftarrow \mathrm{SpectralFilter}(\pmb{z}, M)$

\State 4: Reconstruct masked image from masked latent:
\State \hspace{2em} $\hat{\pmb{x}}_{M} \leftarrow D(\pmb{z}_{M})$

\State 5: Enforce frequency-aligned decoding:
\State \hspace{2em} $\mathcal{L}_{\mathrm{DSM}} \leftarrow \|\hat{\pmb{x}}_{M} - \pmb{x}_{M}\|_1$

\State \Return $\mathcal{L}_{\mathrm{DSM}}$
\end{algorithmic}
\end{algorithm}
\end{minipage}
\end{figure}
\endgroup

\paragraph{Decoding Spectrum Matching (DSM).}
While ESM shapes the latent spectrum by regularizing the encoder $E(\cdot)$, DSM enforces decoder-side frequency alignment between the latent and image. As shown in Algorithm~\ref{alg:dsm}, we again start from the latent $\pmb{z}=E(\pmb{x})$, and then sample a shared frequency mask $M\sim\mathcal{M}$. In practice, we apply the triangular mask $M$ (with shape \tikz\draw[fill=black] (0,0) -- (0.25,0) -- (0.25,0.25) -- cycle;) in the 2D-DCT block \citep{ahmed1974discrete} where the high frequencies are on the bottom-right corner. Therefore, the mask acts as a low-pass filter, preserving only a subset of low-frequency components and suppressing high-frequency components. We then apply the same blockwise DCT mask to both the image $\pmb{x}$ and the latent $\pmb{z}$:
\[
\pmb{x}_{M}=\mathrm{SpectralFilter}(\pmb{x},M), \qquad
\pmb{z}_{M}=\mathrm{SpectralFilter}(\pmb{z},M).
\]
Finally, the decoder is trained to reconstruct the masked image from the masked latent $\hat{\pmb{x}}_{M}=D(\pmb{z}_{M})$, and the DSM loss is defined as an $\ell_1$ reconstruction objective $\mathcal{L}_{\mathrm{DSM}}=\|\hat{\pmb{x}}_{M}-\pmb{x}_{M}\|_1$. In practice, we use the compound loss below to train an Autoencoder:

\noindent
\begin{equation}
\label{eq: VAE-DSM loss}
\mathcal{L}_{\mathrm{DSM-AE}}(\pmb{x})
= \mathcal{L}_{\mathrm{DSM}} + 
\lambda_1 \mathcal{L}_{\mathrm{LPIPS}}(\pmb{x}_{M},
\hat{\pmb{x}}_{M}) + \lambda_2 \mathcal{L}_{\mathrm{GAN}}(\pmb{x}_{M}, \hat{\pmb{x}}_{M}).
\end{equation}
\noindent

Note that during training, the sampled mask $M$ may also be empty (i.e., no filtering), in which case all frequency components of $\pmb{x}$ and $\pmb{z}$ are preserved. Under this setting, Eq.~\eqref{eq: VAE-DSM loss} reduces to the standard VAE reconstruction loss. In Section~\ref{subsec: Results of ESM and DSM in VAE}, we will show that both ESM and DSM can achieve improved diffusion results compared with standard VAEs.

\subsection{A Unified Spectral Interpretation of Prior Work}
\label{subsec: Spectrum Matching Unifies Prior Observations and Approaches}

Beyond the empirical performance gains, a key advantage of Spectrum Matching is that it provides a unified lens for understanding prior observations and methods in the recent VAE literature.

\paragraph{Explaining over-noisy or over-smoothed latents.}
Several works report that the latent space of SD-VAE contains overly strong high-frequency components \citep{skorokhodov2025improving}, and that these high-frequency bands can even carry substantial low-frequency semantic information from the RGB image \citep{lai2025toward}. These observations motivate our hypothesis that latent spectral organization and decoder frequency correspondence matter for diffusability, informed by the pixel-space analysis in Section~\ref{subsec: Power-Law PSD and Denoising MSE Reduction} and prior evidence on high-frequency denoising \citep{falck2025fourier}. The spectral view provides a possible interpretation: the flattening tendency in the latent spectrum can be interpreted as a result of entropy maximization during compression (see Lemma~\ref{lemma:1}); however, a standard VAE may overuse this mechanism and shift too much information into high-frequency bands. ESM directly counteracts this behavior by matching the latent PSD to a flattened but still power-law target, while DSM further prevents semantic drift across frequency bands by enforcing frequency-consistent decoding. We present in Section~\ref{subsec: Results of ESM and DSM in VAE} that both ESM and DSM can solve this excessive-high-frequency issue and improve downstream diffusion quality. As pointed out by \citet{liu2025delving,leng2025repa}, the opposite extreme is also problematic: an overly smooth latent space is not ideal for diffusion modeling either. If the latent over-concentrates energy in low frequencies, the representation becomes too lossy and fails to preserve sufficient image detail. In our framework, ESM avoids both extremes by explicitly regularizing the latent toward a flattened power-law PSD.

\paragraph{A unified interpretation of recent methods.}
Spectrum Matching provides a unified perspective on several recent approaches through encoder-side spectrum regularization and decoder-side frequency alignment.
On the encoder side, UAE \citep{fan2025prism} aligns low-frequency latent DCT coefficients with those of a frozen semantic teacher, such as DINOv2 \citep{oquab2024dinov2}.
VA-VAE \citep{yao2025reconstruction} aligns latent features with frozen DINOv2 features through a learned linear projection.
Our spectral analysis (Appendix~\ref{append: Analysis of Low Frequency Alignment in UAE}) shows that \textbf{DINOv2 features approximately follow a flattened power-law PSD} $S(\omega)\propto|\omega|^{-(\alpha-\delta)}$ with $\delta\approx1.0$ relative to the input image spectrum, and that VA-VAE latents exhibit a PSD close to this target.
These observations support interpreting DINOv2-based alignment as implicit spectral regularization related to ESM.
PAE \citep{yue2026matters} offers a related connection through spatial Gram alignment, which matches frequency-wise energy and cross-frequency
relations of normalized latent and teacher features. 
On the decoder side, Scale Equivariance \citep{skorokhodov2025improving} and the scaling component of EQ-VAE \citep{kouzelis2025eq} encourage reconstruction consistency under downsampling. \textbf{The low-pass filtering underlying downsampling can be viewed as a particular form of spectral masking} according to \citet{ning2025dctdiff}. DSM generalizes this filtering principle by applying shared spectral masks to latents and images, allowing more flexible low-pass patterns (see Appendix~\ref{append: Spectral Connections to Scale Equivariance and EQ-VAE}). Table~\ref{tab: Spectrum Matching Unifies Prior Observations and Approaches} summarizes these connections.

\noindent
\begin{table}[htb]
\vskip -0.15in
\caption{
A Unified Spectral Interpretation of Prior Approaches and Empirical Observations.}
\label{tab: Spectrum Matching Unifies Prior Observations and Approaches}
\centering
\small
\begin{tabularx}{\linewidth}{@{} L{0.18\linewidth} L{0.4\linewidth} >{\raggedright\arraybackslash}X @{}}
\toprule
\multicolumn{1}{c}{Prior research} & Observations / Method & Spectrum Matching View 
\\
\midrule
Scale Equivariance \citep{skorokhodov2025improving} &
Observation: standard VAE latents exhibit excessive high-frequency
energy relative to RGB images. &
ESM regularizes excessive high-frequency energy toward a flattened
power-law PSD. \\
\midrule
FreqWarm \citep{lai2025toward} &
Observation: low-frequency image information can be encoded in
high-frequency latent components. & DSM encourages correspondence between latent and image frequency bands.
\\
\midrule
SSVAE \citep{liu2025delving} &
Observation: non-uniform latent spectra improve diffusability
compared with uniform spectra. & ESM ensures that the  latent has a power-law spectrum.
\\
\midrule
REPA-E \citep{leng2025repa} & Observation: over-smoothed latent space leads to bad diffusion modeling. & ESM ensures that the  latent has a power-law spectrum.
\\
\midrule
UAE \citep{fan2025prism} &
Method: aligns low-frequency latent DCT coefficients with those
of a frozen semantic teacher (e.g., DINOv2). & The teacher provides a spectral reference for low-frequency latent bands, related to ESM.
\\
\midrule
VA-VAE \citep{yao2025reconstruction} &
Method: aligns linearly projected latents with frozen DINOv2
features. & The measured latent PSD approximates ESM's flattened
power-law target.
\\
\midrule
PAE (SSR) \citep{yue2026matters} &
Method: aligns latent spatial Gram matrices with refined
VFM targets. &
Spectral matching of normalized features, related to ESM. \\
\midrule
Scale Equivariance \citep{skorokhodov2025improving} &
Method: trains the decoder to reconstruct downsampled images from correspondingly downsampled latents. &
Low-frequency reconstruction consistency through downsampling; DSM extends the filtering principle with shared spectral masks. \\
\midrule
EQ-VAE \citep{kouzelis2025eq} &
Method: regularizes reconstruction consistency under scaling and rotation. &
Its scaling component shares DSM's principle of low-frequency reconstruction consistency. \\
\bottomrule
\end{tabularx}
\end{table}
\noindent

\subsection{Directional Spectral Energy Matters in REPA}
\label{subsec: Directional Spectrum Energy Matters in REPA}

Spectrum considerations are not limited to VAE latents. They also help clarify the representation alignment objective in REPA: what properties should a target representation have to serve as an effective alignment signal? Recent work iREPA \citep{singh2025matters} argues that the spatial structure of the target representation is crucial for REPA, and empirically finds that the RMS Spatial Contrast (RMSC) of the target feature correlates strongly with diffusion generation quality. 

\begin{wrapfigure}{r}{0.4\textwidth}
  \vskip -0.2in
  \includegraphics[width=0.4 \textwidth]{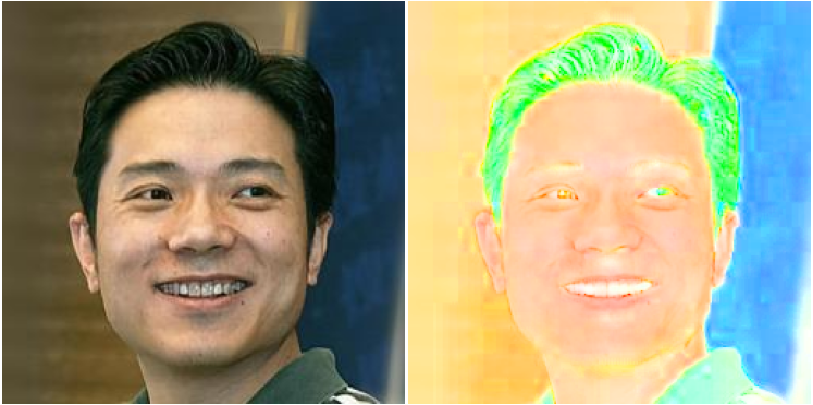}
  \captionsetup{skip=1pt}
  \caption{Directional image (right) maintains the spatial structure of the original image (left).}
\label{fig: directional_image}
\vskip -0.2in
\end{wrapfigure}

\noindent
We observe that the RMSC used in iREPA is mathematically equivalent to the directional spectral energy of the target representation (Proposition~\ref{proposition:2}). Hence, iREPA’s finding can be restated in spectral terms: \textbf{directional spectral energy of the target representation matters for REPA}. Here, the directional field refers to the direction of feature tokens obtained via magnitude normalization. Figure~\ref{fig: directional_image} provides an intuition by visualizing an RGB image after per-pixel magnitude normalization: although absolute magnitudes are removed, the spatial layout remains clearly visible. This visualization is consistent with signal processing studies, which have shown that the phase/direction largely determines the spatial structure \citep{oppenheim2005importance, hassen2013image}.

\vspace{\baselineskip}
\WFclear

\begin{proposition}[\textbf{RMSC is equivalent to directional spectral energy}]
\label{prop:rmsc_nonDC_directional_energy}
Let $\{x_t\}_{t=1}^T \subset \mathbb{R}^D$ be token features with $\|x_t\|_2>0$. Define the normalized tokens
$u_t := x_t/\|x_t\|_2$ and their mean $\bar u := \frac{1}{T}\sum_{t=1}^T u_t$.
Let $\mathrm{RMSC}(x):=\sqrt{\frac{1}{T}\sum_{t=1}^T \|u_t-\bar u\|_2^2}$ be the normalized RMSC.
Let $\{U_k\}_{k=0}^{T-1}$ be the coefficients of applying an \emph{orthonormal} DCT (along the token index $t$)
to each feature dimension of $\{u_t\}$, i.e., $U_k\in\mathbb{R}^D$ is the DCT coefficient vector at frequency $k$.
Then

\noindent
\begin{equation}
\label{eq: RMSC}
\mathrm{RMSC}(x)^2
\;=\;
\frac{1}{T}\sum_{k=1}^{T-1}\|U_k\|_2^2,
\end{equation}
\noindent

i.e., $T \times\,\mathrm{RMSC}(x)^2$ equals the non-DC spectral energy of the direction field. Proof is in Appendix~\ref{append: Proof of Proposition 2}.

\label{proposition:2}
\end{proposition}

\paragraph{From DC removal to band-pass filter.}
To increase the spatial contrast of the target representation, iREPA applies a spatial normalization $\hat{Z} \;=\; \frac{Z - \alpha\,\mathrm{mean}(Z)}{\mathrm{std}(Z)+\varepsilon}$. We notice that the mean subtraction removes the DC component, thus we propose a more general frequency method: a Difference-of-Gaussians (DoG) filter, which acts as a band-pass operator and can suppress a broader range of low-frequency components beyond DC while also attenuating very high frequencies. Concretely, we replace the above spatial normalization with:

\noindent
\begin{equation}
\label{eq: DoG}
\mathrm{DoG}(Z) \;=\; (G_{\sigma_1} * Z) \;-\; (G_{\sigma_2} * Z),
\qquad 
\hat{Z} \;=\; \mathrm{DoG}(Z)\big/(\mathrm{std}(Z)+\varepsilon),
\end{equation}
\noindent

where $\sigma_2 > \sigma_1$ and $G_{\sigma_1}, G_{\sigma_2}$ are Gaussian kernels. In Section~\ref{subsec: Results of DoG in REPA}, we show that replacing spatial normalization with DoG yields better generation quality than iREPA.

\section{Experiments}
\label{sec: Experiments}

To evaluate the effectiveness of Spectrum Matching, we build Spectrum Matching Autoencoders upon SD-VAE \citep{rombach2022high} while retaining its original U-Net architecture. For convenience, we refer to the autoencoders trained with our ESM and DSM regularizers as \textbf{ESM-AE} and \textbf{DSM-AE}, respectively. We evaluate reconstruction quality using reconstruction Fréchet Inception Distance (rFID) \citep{heusel2017gans}, Peak Signal-to-Noise Ratio (PSNR), and Structural Similarity (SSIM) \citep{wang2004image}, and assess generation quality using gFID. To ensure a fair comparison, we use the same network architecture and training protocol for all reported models. We further evaluate our methods in two settings: training from scratch and fine-tuning a pretrained VAE. Architectural details and the full set of training hyperparameters are provided in Appendix~\ref{append: Training Parameters}. All ablation studies are deferred to Appendix~\ref{append: Ablation Studies of ESM and DSM}.

\subsection{Training VAEs from Scratch with ESM and DSM}
\label{subsec: Results of ESM and DSM in VAE}

First, we train SD-VAE \citep{rombach2022high}, Scale Equivariance \citep{skorokhodov2025improving}, ESM-AE and DSM-AE from scratch on the CelebA $256\times256$ dataset  \citep{karras2018progressive}. All models share the same network architecture and are trained with batch size 48 for 500k steps, which we find sufficient for convergence. Second, we apply the diffusion transformer U-ViT \citep{bao2022all} to learn the latent distributions. Since there is no clear correlation between reconstruction quality (rFID) and generation quality (gFID) \citep{hansenlearnings}, we select the three VAE checkpoints with the lowest rFID for each method and train a separate diffusion model on the latents produced by each checkpoint. We report the results corresponding to the diffusion model with the lowest gFID among these runs. We compute gFID using 50k samples generated by a 100-step DDIM sampler \citep{songdenoising}. For SD-VAE, we use the objective in Eq.~\eqref{eq: SDVAE loss}. For Scale Equivariance, we follow \citet{skorokhodov2025improving} and apply the same objective while randomly downsampling both the target image and latent size by $1\times$ (i.e., unchanged), $2\times$, or $4\times$. ESM-AE is trained with Eq.~\eqref{eq: VAE-ESM loss} using $\beta=0.01$ and $\delta=1.0$. DSM-AE is trained with Eq.~\eqref{eq: VAE-DSM loss}; details of the frequency mask family $\mathcal{M}$ are provided in Appendix~\ref{append: Design of Frequency Mask}. In our experiments, we treat both SD-VAE and Scale Equivariance as the baselines, and compare our Spectrum Matching models against them. We consider two commonly used VAE configurations: $f8d4, f16d16$, where $f$ denotes the spatial downsampling ratio and $d$ means the depth of the latent. Results in Table~\ref{tab: Reconstruction and Generation Results on CelebA} highlight that ESM-AE and DSM-AE consistently outperform SD-VAE regarding the generation quality (gFID) with the faster training of diffusion modeling. Scale Equivariance also improves gFID over SD-VAE, while DSM-AE achieves lower gFID in the evaluated settings. Regarding the reconstruction quality, all models show similar performance.

\begin{table}[htb]
\vskip -0.0in
\caption{
Reconstruction and generation on CelebA $256\times256$. Steps denote diffusion training iterations; bold marks the best gFID.}
\label{tab: Reconstruction and Generation Results on CelebA}
\centering
\small
\setlength{\tabcolsep}{5.5pt}
\begin{tabular}{@{}lllllllllll@{}}
\toprule
\multicolumn{1}{c}{\multirow{2}{*}{Models}} & \multicolumn{5}{c}{f8d4} & \multicolumn{5}{c}{f16d16} \\ \cmidrule(lr){2-6} \cmidrule(lr){7-11} 
 & rFID $\downarrow$ & PSNR $\uparrow$ & SSIM $\uparrow$ & steps & gFID $\downarrow$ & rFID $\downarrow$ & PSNR $\uparrow$ & SSIM $\uparrow$ & steps & gFID $\downarrow$ \\ \midrule
SD-VAE & 1.31 & 31.23 & 0.788 & 125k & 6.63 & 1.72 & 30.08 & 0.753 & 150k & 6.88 \\ \midrule
Scale Equivariance & 1.24 & 30.98 & 0.781 & 100k & 5.87 & 1.31 & 29.70 & 0.741 & 100k &5.46 \\ \midrule
ESM-AE & 1.55 & 31.63 & 0.793 & 100k & 5.91 & 1.56 & 31.30 & 0.782 & 100k & 5.84 \\ \midrule
  \rowcolor{citeblue!10}
DSM-AE & 1.44 & 30.07 & 0.760 & 100k & \textbf{4.44} & 1.39 & 29.32 & 0.745 & 100k & 4.49 \\ \midrule
\rowcolor{citeblue!10}
ESM+DSM & 1.37 & 30.04 & 0.765 & 100k & 4.48 & 1.31 & 29.45 & 0.745 & 100k & \textbf{4.38} \\ \bottomrule
\end{tabular}
\vskip -0.0in
\end{table}

We further analyze the latent space and its spectrum to understand the practical effects of ESM and DSM. As shown in Figure~\ref{fig: latent_PCA}, SD-VAE produces noticeably noisier latents, whereas ESM-AE and DSM-AE yield smoother latents while preserving semantic structure. This behavior is reflected more clearly in the spectral domain (Figure~\ref{fig:spectrum_comparison}(a,b)): both ESM and DSM induce a flatter power-law PSD than the RGB image, consistent with our hypothesis. While Scale Equivariance also exhibits a power-law-like spectrum, its latent PSD is less smooth and less consistently regularized than that of DSM-AE in this experiment.

\noindent
\begin{figure}[htb]
\vskip -0.1in
\begin{center}
\centerline{\includegraphics[width=0.8\columnwidth]{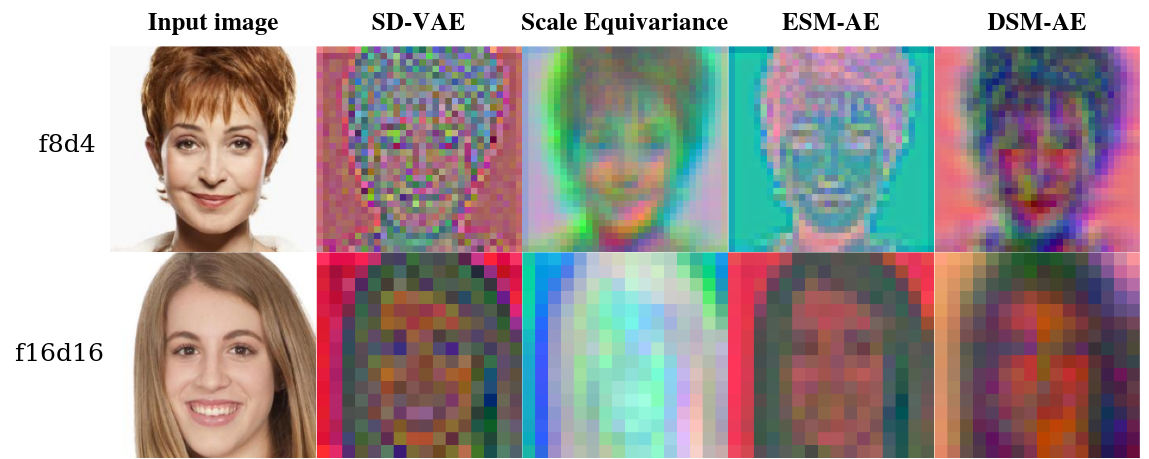}}
\caption{PCA visualization (top three principal components) of the latent space of different autoencoders. 
}
\label{fig: latent_PCA}
\end{center}
\vskip -0.2in
\end{figure}
\noindent

\begin{figure}[htbp]
  \centering
  \includegraphics[width=\linewidth]{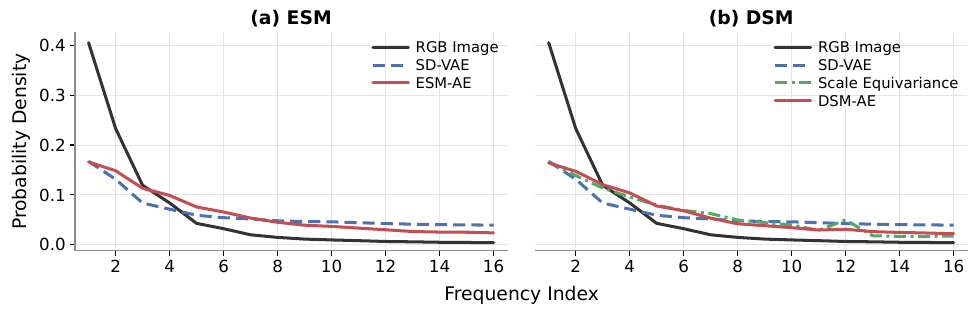}
  \caption{PSD distributions of RGB images and autoencoder latents on CelebA $256\times256$.}
  \label{fig:spectrum_comparison}
\end{figure}
\noindent

Table~\ref{tab: Reconstruction and Generation Results on CelebA} shows that DSM achieves generation performance comparable to ESM+DSM on CelebA $256\times256$. Moreover, DSM alone can induce a flattened power-law PSD (Figure~\ref{fig:spectrum_comparison}(b)). To keep our method simple, we therefore recommend DSM as the primary approach to Spectrum Matching in practice and further benchmark it on ImageNet $256\times256$. We again train both $f16d16$ type SD-VAE and DSM-AE from scratch for 600k steps, with a batch size of 128.  Then, we apply the diffusion Transformer SiT \citep{ma2024sit} to model the latent distribution. To accelerate the diffusion training, we adopt REPA \citep{yu2025repa} and train SiT for 400k steps with a batch size of 256. All gFIDs are measured by 50k samples generated by a 100-step DDIM sampler \citep{songdenoising} and Classifier-Free Guidance \citep{ho2022classifier}. Table~\ref{tab: Reconstruction and Generation Results on ImageNet} indicates that DSM-AE consistently achieves better generation results than SD-VAE under the condition of REPA or without REPA, even though the reconstruction quality is slightly worse than SD-VAE (a similar trend was observed by \citet{skorokhodov2025improving}).

\noindent
\begin{table}[htb]
\vskip -0.1in
\caption{
Reconstruction and generation results on ImageNet $256\times256$ (training from scratch).}
\label{tab: Reconstruction and Generation Results on ImageNet}
\centering
\small
\begin{tabular}{@{}llllllllllll@{}}
\toprule
\multicolumn{1}{c}{\multirow{2}{*}{Models}} & \multicolumn{3}{c}{f16d16} & \multicolumn{4}{c}{gFID (without REPA)} & \multicolumn{4}{c}{gFID (with REPA)} \\ \cmidrule(lr){2-4} \cmidrule(lr){5-8} \cmidrule(lr){9-12} 
\multicolumn{1}{c}{} & rFID $\downarrow$ & PSNR $\uparrow$ & SSIM $\uparrow$ & 100k & 200k & 300k & 400k & 100k & 200k & 300k & 400k \\ \midrule
SD-VAE & 1.20 & 25.44 & 0.659 & 27.43 & 15.88 & 12.29 & 10.62 & 22.25 & 10.52 & 7.79 & 6.83 \\ \midrule
\rowcolor{citeblue!10}
DSM-AE & 1.38 & 24.15 & 0.620 & \textbf{22.87} & \textbf{13.34} & \textbf{10.64} & \textbf{9.55} & \textbf{18.80} & \textbf{8.98} & \textbf{6.92} & \textbf{6.26} \\ \bottomrule
\end{tabular}
\vskip -0.1in
\end{table}
\noindent

\subsection{Fine-Tuning Pretrained VAEs with DSM}
\label{subsec: finetune with DSM}

Beyond training from scratch, DSM can also be used to fine-tune pretrained VAEs to improve their diffusability. We use pretrained SD-VAE ($f8d4$) \citep{rombach2022high} as the baseline and, starting from the same checkpoint, fine-tune separate models with DSM and EQ-VAE \citep{kouzelis2025eq} on ImageNet $256\times256$ for five epochs. We additionally compare against VA-VAE ($f16d32$) \citep{yao2025reconstruction}, using its official checkpoint without further fine-tuning, since it has already been extensively trained on ImageNet $256\times256$. We then train SiT-B/2 and SiT-XL/2 \citep{ma2024sit} with REPA on the resulting latents under the same training settings, except that we use a patch size of 1 for VA-VAE to match the number of diffusion tokens used with SD-VAE. Following the evaluation protocol in Section~\ref{subsec: Results of ESM and DSM in VAE}, we compute gFID from 50k samples generated with a 100-step DDIM sampler and classifier-free guidance (CFG).

Table~\ref{tab: pretrained autoencoder comparison} shows that EQ-VAE, VA-VAE, and DSM-AE all improve generation quality over SD-VAE. Among them, DSM-AE consistently achieves the lowest gFID across all evaluated settings, despite slightly worse reconstruction metrics than other models. In particular, with SiT-XL trained for 200 epochs, DSM-AE reaches a gFID of 1.66, outperforming SD-VAE (1.98), EQ-VAE (1.79), and VA-VAE (1.89). These results highlight the effectiveness of DSM fine-tuning in improving the diffusability of pretrained VAEs.

\begin{table}[htb]
\caption{Reconstruction and generation on ImageNet $256\times256$ (fine-tuning pretrained SD-VAE).}
\label{tab: pretrained autoencoder comparison}
\centering
\small
\begin{tabular}{@{}lccccc@{}}
\toprule
\multirow{2}{*}{Models} & \multirow{2}{*}{rFID $\downarrow$} & \multirow{2}{*}{PSNR $\uparrow$} & \multicolumn{3}{c}{gFID $\downarrow$} \\
\cmidrule(lr){4-6}
 & & & SiT-B (80 epochs) & SiT-XL (100 epochs) & SiT-XL (200 epochs) \\
\midrule
SD-VAE ($f8d4$) & 0.55 & 26.11 & 3.68 & 2.15 & 1.98 \\
EQ-VAE ($f8d4$)& 0.61 & 25.97 & 3.31 & 2.02 & 1.79 \\
VA-VAE ($f16d32$) & 0.28 & 27.71 & 3.39 & 2.08 & 1.89 \\
\rowcolor{citeblue!10}
DSM-AE ($f8d4$)& 0.66 & 25.85 & \textbf{3.20} & \textbf{1.82} & \textbf{1.66} \\
\bottomrule
\end{tabular}
\end{table}

\subsection{Results of DoG on REPA}
\label{subsec: Results of DoG in REPA}

\begin{wraptable}{r}{0.4\textwidth}
\vskip -0.1in
\small
\centering
\captionsetup{skip=2pt}
\caption{gFID on ImageNet $256\times256$ using DINOv2-B as encoder on SiT-B/2.}
\begin{tabular}{@{}lllll@{}}
\toprule
\multicolumn{5}{c}{Classifier-Free Guidance = 1.8} \\
\midrule
Models & 100k & 200k & 300k & 400k \\
\midrule
REPA     & 23.92 & 10.67 & 7.13 & 5.68 \\
iREPA    & \textbf{17.73} & \textbf{8.51} & 6.23 & 5.07 \\
REPA-DoG & 18.52 & 8.60  & \textbf{6.15} & \textbf{4.98} \\
\bottomrule
\end{tabular}
\label{tab:dog_results}
\vskip -0.0in
\end{wraptable}

Building on iREPA's findings and our spectral interpretation, we investigate DoG preprocessing for representation alignment. To verify the effectiveness of DoG, we utilize the codebase of REPA \citep{yu2025repa}, and train REPA, iREPA \citep{singh2025matters}, and our REPA-DoG on ImageNet $256\times256$ for 400k steps with SiT-B/2 as the diffusion backbone. All methods use the same recommended training configuration from \citet{yu2025repa} (e.g., alignment coefficient $=0.5$) to ensure a controlled comparison. Table~\ref{tab:dog_results} reports gFID throughout diffusion training and shows that REPA-DoG overtakes REPA and iREPA as training proceeds, achieving the best gFID (4.98) at 400k steps using CFG \citep{ho2022classifier}. Following iREPA \citep{singh2025matters}, we visualize token-wise cosine similarity maps using DINOv3-B (Figure~\ref{fig: norm_spatial_structure}) to illustrate the impact of spatial normalization used in iREPA and our DoG method. The stronger contrast in these similarity maps is consistent with the role of spatial structure reported by \citet{singh2025matters}.

\noindent
\vskip -0.2in
\begin{figure}[!h]
    \centering    
    \includegraphics[width=1.0\columnwidth]{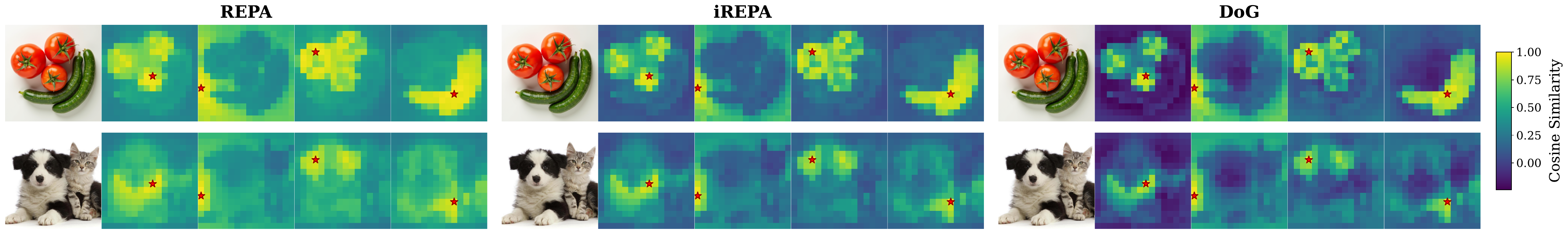}
    \caption{Impact of processing the feature representation of DINOv3-B \citep{simeoni2025dinov3}. Comparison of token-wise cosine similarity maps. The red star indicates the reference token, while the heatmap represents its affinity with the spatial query tokens.}
    \label{fig: norm_spatial_structure}
\end{figure}
\vskip -0.2in
\noindent

\vskip 0.4in
\section{Conclusion}
\label{sec: Conclusion}

In this paper, we propose Spectrum Matching as a unified perspective on latent diffusability, formalized by the Spectrum Matching Hypothesis. Our motivation is to organize latent spectral energy and frequency correspondence so that the low-frequency denoising advantage can support perceptually meaningful semantic generation. We instantiate this principle with two practical mechanisms, ESM and DSM. Experiments on CelebA and ImageNet demonstrate improved generation quality over SD-VAE and the prior method. We further extend the spectral perspective to REPA by linking RMSC to directional spectral energy and introducing a DoG-based band-pass preprocessing that yields additional gains. Our study focuses on image VAEs given the limited computational resources. A key limitation is that we did not investigate Spectrum Matching in video autoencoders, where temporal frequency structure and spatiotemporal coupling may introduce new constraints and opportunities. We leave this to our future work.

\section{Acknowledgement}
\label{sec: Acknowledgement}

We acknowledge the EuroHPC Joint Undertaking for awarding this project access to the EuroHPC supercomputer LEONARDO, hosted by CINECA (Italy) and the LEONARDO consortium through a EuroHPC Regular Access call. We thank Long Zhao (DeepMind) for providing suggestions on training VAEs.

\bibliography{bibi}

@misc{dieleman2024spectral,
  author = {Dieleman, Sander},
  title = {Diffusion is spectral autoregression},
  url = {https://sander.ai/2024/09/02/spectral-autoregression.html},
  year = {2024}
}

@inproceedings{bonnaire2025diffusion,
  title={Why Diffusion Models Don't Memorize: The Role of Implicit Dynamical Regularization in Training},
  author={Bonnaire, Tony and Urfin, Rapha{\"e}l and Biroli, Giulio and M{\'e}zard, Marc},
  booktitle={NeurIPS},
  year={2025}
}

@article{falck2025fourier,
  title={A {Fourier} space perspective on diffusion models},
  author={Falck, Fabian and Pandeva, Teodora and Zahirnia, Kiarash and Lawrence, Rachel and Turner, Richard and Meeds, Edward and Zazo, Javier and Karmalkar, Sushrut},
  journal={arXiv preprint arXiv:2505.11278},
  year={2025}
}

@inproceedings{rombach2022high,
  title={High-resolution image synthesis with latent diffusion models},
  author={Rombach, Robin and Blattmann, Andreas and Lorenz, Dominik and Esser, Patrick and Ommer, Bj{\"o}rn},
  booktitle={CVPR},
  pages={10684--10695},
  year={2022}
}

@inproceedings{skorokhodov2025improving,
  title={Improving the diffusability of autoencoders},
  author={Skorokhodov, Ivan and Girish, Sharath and Hu, Benran and Menapace, Willi and Li, Yanyu and Abdal, Rameen and Tulyakov, Sergey and Siarohin, Aliaksandr},
  booktitle={ICML},
  year={2025}
}

@inproceedings{kouzelis2025eq,
  title={{EQ-VAE}: Equivariance regularized latent space for improved generative image modeling},
  author={Kouzelis, Theodoros and Kakogeorgiou, Ioannis and Gidaris, Spyros and Komodakis, Nikos},
  booktitle={ICML},
  year={2025}
}

@article{lai2025toward,
  title={Toward Diffusible High-Dimensional Latent Spaces: A Frequency Perspective},
  author={Lai, Bolin and Wang, Xudong and Rambhatla, Saketh and Rehg, James M and Kira, Zsolt and Girdhar, Rohit and Misra, Ishan},
  journal={arXiv preprint arXiv:2511.22249},
  year={2025}
}

@article{liu2025delving,
  title={Delving into latent spectral biasing of video {VAEs} for superior diffusability},
  author={Liu, Shizhan and Deng, Xinran and Yang, Zhuoyi and Teng, Jiayan and Gu, Xiaotao and Tang, Jie},
  journal={arXiv preprint arXiv:2512.05394},
  year={2025}
}

@article{fan2025prism,
  title={The Prism Hypothesis: Harmonizing Semantic and Pixel Representations via Unified Autoencoding},
  author={Fan, Weichen and Diao, Haiwen and Wang, Quan and Lin, Dahua and Liu, Ziwei},
  journal={arXiv preprint arXiv:2512.19693},
  year={2025}
}

@inproceedings{leng2025repa,
  title={{REPA-E}: Unlocking {VAE} for end-to-end tuning of latent diffusion transformers},
  author={Leng, Xingjian and Singh, Jaskirat and Hou, Yunzhong and Xing, Zhenchang and Xie, Saining and Zheng, Liang},
  booktitle={ICCV},
  pages={18262--18272},
  year={2025}
}

@article{oquab2024dinov2,
  title={{DINOv2}: Learning Robust Visual Features without Supervision},
  author={Oquab, Maxime and Darcet, Timoth{\'e}e and Moutakanni, Th{\'e}o and Vo, Huy and Szafraniec, Marc and Khalidov, Vasil and Fernandez, Pierre and Haziza, Daniel and Massa, Francisco and El-Nouby, Alaaeldin and others},
  journal={TMLR},
  year={2024}
}

@inproceedings{ning2025dctdiff,
  title={{DCTdiff}: Intriguing Properties of Image Generative Modeling in the {DCT} Space},
  author={Ning, Mang and Li, Mingxiao and Su, Jianlin and Haozhe, Jia and Liu, Lanmiao and Benes, Martin and Chen, Wenshuo and Salah, Albert Ali and Ertugrul, Itir Onal},
  booktitle={ICML},
  pages={46498--46524},
  year={2025},
  organization={PMLR}
}

@inproceedings{singh2025matters,
  title={What matters for Representation Alignment: Global Information or Spatial Structure?},
  author={Singh, Jaskirat and Leng, Xingjian and Wu, Zongze and Zheng, Liang and Zhang, Richard and Shechtman, Eli and Xie, Saining},
  booktitle={ICLR},
  year={2026}
}

@article{wang2004image,
  title={Image quality assessment: from error visibility to structural similarity},
  author={Wang, Zhou and Bovik, Alan C and Sheikh, Hamid R and Simoncelli, Eero P},
  journal={IEEE Transactions on Image Processing},
  volume={13},
  number={4},
  pages={600--612},
  year={2004},
  publisher={IEEE}
}

@inproceedings{heusel2017gans,
  title={{GANs} trained by a two time-scale update rule converge to a local {Nash} equilibrium},
  author={Heusel, Martin and Ramsauer, Hubert and Unterthiner, Thomas and Nessler, Bernhard and Hochreiter, Sepp},
  booktitle={NeurIPS},
  year={2017}
}

@inproceedings{bao2022all,
  title={All are Worth Words: A {ViT} Backbone for Diffusion Models},
  author={Bao, Fan and Nie, Shen and Xue, Kaiwen and Cao, Yue and Li, Chongxuan and Su, Hang and Zhu, Jun},
  booktitle = {CVPR},
  year={2023}
}

@inproceedings{karras2018progressive,
  title={Progressive Growing of {GANs} for Improved Quality, Stability, and Variation},
  author={Karras, Tero and Aila, Timo and Laine, Samuli and Lehtinen, Jaakko},
  booktitle={ICLR},
  year={2018}
}

@inproceedings{songdenoising,
  title={Denoising Diffusion Implicit Models},
  author={Song, Jiaming and Meng, Chenlin and Ermon, Stefano},
  booktitle={ICLR},
  year={2021}
}

@inproceedings{hansenlearnings,
  title={Learnings from Scaling Visual Tokenizers for Reconstruction and Generation},
  author={Hansen-Estruch, Philippe and Yan, David and Chuang, Ching-Yao and Zohar, Orr and Wang, Jialiang and Hou, Tingbo and Xu, Tao and Vishwanath, Sriram and Vajda, Peter and Chen, Xinlei},
  booktitle={ICML},
  year={2025}
}

@inproceedings{ma2024sit,
  title={{SiT}: Exploring Flow and Diffusion-Based Generative Models with Scalable Interpolant Transformers},
  author={Ma, Nanye and Goldstein, Mark and Albergo, Michael S and Boffi, Nicholas M and Vanden-Eijnden, Eric and Xie, Saining},
  booktitle={ECCV},
  pages={23--40},
  year={2024}
}

@inproceedings{yu2025repa,
  title={Representation Alignment for Generation: Training Diffusion Transformers Is Easier Than You Think},
  author={Sihyun Yu and Sangkyung Kwak and Huiwon Jang and Jongheon Jeong and Jonathan Huang and Jinwoo Shin and Saining Xie},
  year={2025},
  booktitle={ICLR},
}

@inproceedings{esser2024scaling,
  title={Scaling rectified flow transformers for high-resolution image synthesis},
  author={Esser, Patrick and Kulal, Sumith and Blattmann, Andreas and Entezari, Rahim and M{\"u}ller, Jonas and Saini, Harry and Levi, Yam and Lorenz, Dominik and Sauer, Axel and Boesel, Frederic and others},
  booktitle={ICML},
  year={2024}
}

@article{labs2025flux,
  title={{FLUX.1 Kontext}: Flow Matching for In-Context Image Generation and Editing in Latent Space},
  author={Labs, Black Forest and Batifol, Stephen and Blattmann, Andreas and Boesel, Frederic and Consul, Saksham and Diagne, Cyril and Dockhorn, Tim and English, Jack and English, Zion and Esser, Patrick and others},
  journal={arXiv preprint arXiv:2506.15742},
  year={2025}
}

@inproceedings{podellsdxl,
  title={{SDXL}: Improving Latent Diffusion Models for High-Resolution Image Synthesis},
  author={Podell, Dustin and English, Zion and Lacey, Kyle and Blattmann, Andreas and Dockhorn, Tim and M{\"u}ller, Jonas and Penna, Joe and Rombach, Robin},
  booktitle={ICLR},
  year={2024}
}

@article{medi2025missing,
  title={Missing Fine Details in Images: Last Seen in High Frequencies},
  author={Medi, Tejaswini and Wang, Hsien-Yi and Rampini, Arianna and Keuper, Margret},
  journal={arXiv preprint arXiv:2509.05441},
  year={2025}
}

@inproceedings{chen2025dc,
  title={{DC-AE} 1.5: Accelerating diffusion model convergence with structured latent space},
  author={Chen, Junyu and Zou, Dongyun and He, Wenkun and Chen, Junsong and Xie, Enze and Han, Song and Cai, Han},
  booktitle={ICCV},
  pages={19628--19637},
  year={2025}
}

@article{team2025nextstep,
  title={{NextStep-1}: Toward autoregressive image generation with continuous tokens at scale},
  author={Team, NextStep and Han, Chunrui and Li, Guopeng and Wu, Jingwei and Sun, Quan and Cai, Yan and Peng, Yuang and Ge, Zheng and Zhou, Deyu and Tang, Haomiao and others},
  journal={arXiv preprint arXiv:2508.10711},
  year={2025}
}

@inproceedings{ning2023input,
  title={Input perturbation reduces exposure bias in diffusion models},
  author={Ning, Mang and Sangineto, Enver and Porrello, Angelo and Calderara, Simone and Cucchiara, Rita},
  booktitle={ICML},
  year={2023}
}

@inproceedings{ning2023elucidating,
  title={Elucidating the exposure bias in diffusion models},
  author={Ning, Mang and Li, Mingxiao and Su, Jianlin and Salah, Albert Ali and Ertugrul, Itir Onal},
  booktitle={ICLR},
  year={2024}
}

@inproceedings{zhaoepsilon,
  title={{Epsilon-VAE}: Denoising as Visual Decoding},
  author={Zhao, Long and Woo, Sanghyun and Wan, Ziyu and Li, Yandong and Zhang, Han and Gong, Boqing and Adam, Hartwig and Jia, Xuhui and Liu, Ting},
  booktitle={ICML},
  year={2025}
}

@inproceedings{yao2025reconstruction,
  title={Reconstruction vs. generation: Taming optimization dilemma in latent diffusion models},
  author={Yao, Jingfeng and Yang, Bin and Wang, Xinggang},
  booktitle={CVPR},
  pages={15703--15712},
  year={2025}
}

@article{wu2025qwen,
  title={{Qwen-Image} technical report},
  author={Wu, Chenfei and Li, Jiahao and Zhou, Jingren and Lin, Junyang and Gao, Kaiyuan and Yan, Kun and Yin, Sheng-ming and Bai, Shuai and Xu, Xiao and Chen, Yilei and others},
  journal={arXiv preprint arXiv:2508.02324},
  year={2025}
}

@article{cai2025z,
  title={{Z-Image}: An efficient image generation foundation model with single-stream diffusion transformer},
  author={Cai, Huanqia and Cao, Sihan and Du, Ruoyi and Gao, Peng and Hoi, Steven and Hou, Zhaohui and Huang, Shijie and Jiang, Dengyang and Jin, Xin and Li, Liangchen and others},
  journal={arXiv preprint arXiv:2511.22699},
  year={2025}
}

@article{wan2025wan,
  title={{Wan}: Open and advanced large-scale video generative models},
  author={Wan, Team and Wang, Ang and Ai, Baole and Wen, Bin and Mao, Chaojie and Xie, Chen-Wei and Chen, Di and Yu, Feiwu and Zhao, Haiming and Yang, Jianxiao and others},
  journal={arXiv preprint arXiv:2503.20314},
  year={2025}
}

@article{wu2025hunyuanvideo,
  title={{HunyuanVideo} 1.5 technical report},
  author={Wu, Bing and Zou, Chang and Li, Changlin and Huang, Duojun and Yang, Fang and Tan, Hao and Peng, Jack and Wu, Jianbing and Xiong, Jiangfeng and Jiang, Jie and others},
  journal={arXiv preprint arXiv:2511.18870},
  year={2025}
}

@inproceedings{peebles2023scalable,
  title={Scalable diffusion models with transformers},
  author={Peebles, William and Xie, Saining},
  booktitle={ICCV},
  pages={4195--4205},
  year={2023}
}

@article{ma2025step,
  title={{Step-Video-T2V} technical report: The practice, challenges, and future of video foundation model},
  author={Ma, Guoqing and Huang, Haoyang and Yan, Kun and Chen, Liangyu and Duan, Nan and Yin, Shengming and Wan, Changyi and Ming, Ranchen and Song, Xiaoniu and Chen, Xing and others},
  journal={arXiv preprint arXiv:2502.10248},
  year={2025}
}

@inproceedings{ho2020denoising,
  title={Denoising diffusion probabilistic models},
  author={Ho, Jonathan and Jain, Ajay and Abbeel, Pieter},
  booktitle={NeurIPS},
  volume={33},
  pages={6840--6851},
  year={2020}
}

@inproceedings{songscore,
  title={Score-Based Generative Modeling through Stochastic Differential Equations},
  author={Song, Yang and Sohl-Dickstein, Jascha and Kingma, Diederik P and Kumar, Abhishek and Ermon, Stefano and Poole, Ben},
  booktitle={ICLR},
  year={2020}
}

@inproceedings{liuflow,
  title={Flow Straight and Fast: Learning to Generate and Transfer Data with Rectified Flow},
  author={Liu, Xingchao and Gong, Chengyue and others},
  booktitle={ICLR},
  year={2023}
}

@article{ahmed1974discrete,
  title={Discrete cosine transform},
  author={Ahmed, Nasir and Natarajan, T. and Rao, Kamisetty R},
  journal={IEEE Transactions on Computers},
  volume={100},
  number={1},
  pages={90--93},
  year={1974},
  publisher={IEEE}
}

@article{simeoni2025dinov3,
  title={{DINOv3}},
  author={Sim{\'e}oni, Oriane and Vo, Huy V and Seitzer, Maximilian and Baldassarre, Federico and Oquab, Maxime and Jose, Cijo and Khalidov, Vasil and Szafraniec, Marc and Yi, Seungeun and Ramamonjisoa, Micha{\"e}l and others},
  journal={arXiv preprint arXiv:2508.10104},
  year={2025}
}

@article{ho2022classifier,
  title={Classifier-free diffusion guidance},
  author={Ho, Jonathan and Salimans, Tim},
  journal={arXiv preprint arXiv:2207.12598},
  year={2022}
}

@article{oppenheim2005importance,
  title={The importance of phase in signals},
  author={Oppenheim, Alan V and Lim, Jae S},
  journal={Proceedings of the IEEE},
  volume={69},
  number={5},
  pages={529--541},
  year={2005},
  publisher={IEEE}
}

@article{hassen2013image,
  title={Image sharpness assessment based on local phase coherence},
  author={Hassen, Rania and Wang, Zhou and Salama, Magdy MA},
  journal={IEEE Transactions on Image Processing},
  volume={22},
  number={7},
  pages={2798--2810},
  year={2013},
  publisher={IEEE}
}

@article{wallace1991jpeg,
  title={The {JPEG} still picture compression standard},
  author={Wallace, Gregory K},
  journal={Communications of the ACM},
  volume={34},
  number={4},
  pages={30--44},
  year={1991},
  publisher={AcM New York, NY, USA}
}

@inproceedings{li2026back,
  title={Back to basics: Let denoising generative models denoise},
  author={Li, Tianhong and He, Kaiming},
  booktitle={CVPR},
  pages={36115--36125},
  year={2026}
}

@article{yue2026matters,
  title={What matters for diffusion-friendly latent manifold? prior-aligned autoencoders for latent diffusion},
  author={Yue, Zhengrong and Hu, Taihang and Chen, Mengting and Zhang, Haiyu and Pan, Zihao and Liu, Tao and Wang, Zikang and Lan, Jinsong and Zhu, Xiaoyong and Zheng, Bo and others},
  journal={arXiv preprint arXiv:2605.07915},
  year={2026}
}

@inproceedings{chen2025deep,
  title={Deep compression autoencoder for efficient high-resolution diffusion models},
  author={Chen, Junyu and Cai, Han and Chen, Junsong and Xie, Enze and Yang, Shang and Tang, Haotian and Li, Muyang and Han, Song},
  booktitle={ICLR},
  volume={2025},
  pages={96539--96560},
  year={2025}
}

@inproceedings{zheng2026diffusion,
  title={Diffusion transformers with representation autoencoders},
  author={Zheng, Boyang and Ma, Nanye and Tong, Shengbang and Xie, Saining},
  booktitle={ICLR},
  volume={2026},
  pages={35791--35820},
  year={2026}
}
\bibliographystyle{tmlr}

\newpage
\appendix

\section{Appendix}

\subsection{Proof of Proposition~\ref{proposition:1}}
\label{append: Proof of Proposition}

\paragraph{Setup.} We treat each color channel separately. Let $\pmb{x}_0\in\mathbb{R}^d$ be a random natural image with finite second moments and $y_0(\omega)\triangleq \mathcal{F}(\pmb{x}_0)(\omega)$ be its Fourier coefficients, where $\mathcal{F}$ denotes the unitary discrete Fourier transform (DFT), and let $S(\omega)\triangleq\mathbb{E}|y_0(\omega)|^2$. We make no distributional assumption on $\pmb{x}_0$ beyond finite second moments; in particular, neither Gaussianity nor zero mean is assumed. The noise in the Fourier domain is $\eta\triangleq\mathcal{F}(\pmb{\varepsilon})$. Since $\mathcal{F}$ is unitary and $\pmb{\varepsilon}\sim\mathcal{N}(\pmb{0},\pmb{I})$,
\[
\mathbb{E}[\eta]=\pmb{0},
\qquad
\mathbb{E}[\eta\eta^{H}]=\mathcal{F}\,\mathbb{E}[\pmb{\varepsilon}\pmb{\varepsilon}^{\top}]\,\mathcal{F}^{H}=\pmb{I},
\quad\text{i.e.,}\quad
\mathbb{E}[\eta(\omega)\eta(\omega')^*]=\mathbf{1}[\omega=\omega'],
\]
and $\eta$ is independent of $\pmb{x}_0$. Because $\pmb{\varepsilon}$ is real, $\eta(-\omega)=\eta(\omega)^*$; the proof only uses the second moments above.

\paragraph{Per-frequency decomposition without loss of generality.} For any measurable denoiser $f$ (which may depend on $t$), Parseval's theorem gives
\[
\mathcal{L}_t(f)\triangleq\mathbb{E}\|\pmb{x}_0-f(\pmb{x}_t)\|_2^2
=\sum_{\omega}\ell_{t,\omega}(f),
\qquad
\ell_{t,\omega}(f)\triangleq\mathbb{E}\Big[\big|y_0(\omega)-\mathcal{F}\big(f(\pmb{x}_t)\big)(\omega)\big|^2\Big].
\]
Therefore, the spatial domain MSE objective can be decomposed as a sum of per-frequency MSE terms.
Let $f^\star(\pmb{x}_t)\triangleq\mathbb{E}[\pmb{x}_0\mid\pmb{x}_t]$ be the Bayes-optimal denoiser. By linearity of $\mathcal{F}$ and of conditional expectation, $\mathcal{F}\big(f^\star(\pmb{x}_t)\big)(\omega)=\mathbb{E}[y_0(\omega)\mid\pmb{x}_t]$, which is the minimum-MSE estimator of the scalar $y_0(\omega)$ given the \emph{full} noisy image $\pmb{x}_t$. Hence $\ell_{t,\omega}(f)\ge\ell_{t,\omega}(f^\star)$ holds for every $f$ and every $\omega$ simultaneously, so that
\[
\inf_f\mathcal{L}_t(f)=\sum_{\omega}\inf_f\ell_{t,\omega}(f),
\qquad
\Delta_t(\omega)=S(\omega)-\ell_{t,\omega}(f^\star)=\mathbb{E}\big|\mathbb{E}[y_0(\omega)\mid\pmb{x}_t]\big|^2,
\]
where the last equality follows from the orthogonal decomposition $\mathbb{E}|Y|^2=\mathbb{E}|\mathbb{E}[Y\mid\pmb{x}_t]|^2+\mathbb{E}|Y-\mathbb{E}[Y\mid\pmb{x}_t]|^2$. Thus the per-frequency view does not restrict the denoiser: $\Delta_t(\omega)$ accounts for all spatial-context and cross-frequency information that any denoiser could exploit.

\paragraph{Upper bound.} Since $\ell_{t,\omega}(f)\ge 0$ for every $f$, we have
\[
\Delta_t(\omega)\le S(\omega),
\]
for any distribution of $\pmb{x}_0$ and any denoiser.

\paragraph{Lower bound via the same-frequency linear estimator.} Fix $\omega$ and write $Y_0\triangleq y_0(\omega)$, $Y_t\triangleq y_t(\omega)$, and $\eta\triangleq\eta(\omega)$, so that
\[
Y_t=\sqrt{\bar{\alpha}_t}\,Y_0+\sqrt{1-\bar{\alpha}_t}\,\eta,
\qquad \mathbb{E}|\eta|^2=1,
\]
where the noise variance is independent of $\omega$. Since $\eta$ is zero-mean and independent of $Y_0$, $\mathbb{E}[Y_0\eta^*]=0$, and hence
\[
\mathbb{E}[Y_0Y_t^*]=\sqrt{\bar{\alpha}_t}\,S(\omega),
\qquad
\mathbb{E}|Y_t|^2=\bar{\alpha}_t S(\omega)+(1-\bar{\alpha}_t).
\]
For the same-frequency linear estimator $cY_t$ with $c\in\mathbb{C}$,
\[
\mathbb{E}|Y_0-cY_t|^2
=S(\omega)-2\sqrt{\bar{\alpha}_t}\,S(\omega)\,\mathrm{Re}(c)+|c|^2\big(\bar{\alpha}_t S(\omega)+1-\bar{\alpha}_t\big),
\]
which is minimized by the Wiener coefficient
\[
c_t(\omega)=\frac{\mathbb{E}[Y_0Y_t^*]}{\mathbb{E}|Y_t|^2}=\frac{\sqrt{\bar{\alpha}_t}\,S(\omega)}{\bar{\alpha}_t S(\omega)+(1-\bar{\alpha}_t)},
\qquad
\mathbb{E}|Y_0-c_t(\omega)Y_t|^2=S(\omega)-\frac{\bar{\alpha}_t S(\omega)^2}{\bar{\alpha}_t S(\omega)+(1-\bar{\alpha}_t)}.
\]
Since $c_t(\omega)$ is real and $S(-\omega)=S(\omega)$ for real images, the Wiener filter $f_{\mathrm{W}}(\pmb{x}_t)\triangleq\mathcal{F}^{-1}\big(c_t\odot\mathcal{F}(\pmb{x}_t)\big)$ is a real-valued linear denoiser that realizes $c_t(\omega)Y_t$ at every frequency. Being one admissible denoiser, it yields
\[
\Delta_t(\omega)\;\ge\;S(\omega)-\ell_{t,\omega}(f_{\mathrm{W}})
=\frac{\bar{\alpha}_t S(\omega)^2}{\bar{\alpha}_t S(\omega)+(1-\bar{\alpha}_t)}
=S(\omega)\cdot \frac{\mathrm{SNR}_t(\omega)}{1+\mathrm{SNR}_t(\omega)}
=G_t(\omega),
\]
with $\mathrm{SNR}_t(\omega)\triangleq \frac{\bar{\alpha}_t S(\omega)}{1-\bar{\alpha}_t}$. Only second moments and the independence of the noise are used.

\paragraph{Tightness of the lower bound.} Since $c_t(\omega)Y_t$ is a function of $\pmb{x}_t$ and $Y_0-\mathbb{E}[Y_0\mid\pmb{x}_t]$ is orthogonal to every function of $\pmb{x}_t$, the Pythagorean identity gives
\[
\Delta_t(\omega)-G_t(\omega)
=\ell_{t,\omega}(f_{\mathrm{W}})-\ell_{t,\omega}(f^\star)
=\mathbb{E}\big|\mathbb{E}[y_0(\omega)\mid\pmb{x}_t]-c_t(\omega)\,y_t(\omega)\big|^2\;\ge\;0.
\]
Hence the lower bound is attained if and only if the Bayes-optimal denoiser coincides with the Wiener filter at frequency $\omega$, e.g., when $\pmb{x}_0$ is a zero-mean stationary Gaussian field. For general natural images, the gap quantifies the additional signal that a nonlinear, context-aware denoiser can extract beyond $G_t(\omega)$. Gaussianity is thus merely one sufficient condition for equality and is not needed for the bound itself.

\paragraph{Frequency ordering.} Combining the upper and lower bounds, for any frequencies $\omega_1,\omega_2$,
\[
G_t(\omega_1)\ge S(\omega_2)
\quad\Longrightarrow\quad
\Delta_t(\omega_1)\ge G_t(\omega_1)\ge S(\omega_2)\ge\Delta_t(\omega_2).
\]
For $S(\omega)=K|\omega|^{-\alpha}$ with $\omega\neq0$,
\[
G_t(\omega_1)\ge S(\omega_2)
\iff
\frac{S(\omega_1)}{S(\omega_2)}\ge 1+\frac{1}{\mathrm{SNR}_t(\omega_1)}
\iff
|\omega_2|\ge\Big(1+\frac{1}{\mathrm{SNR}_t(\omega_1)}\Big)^{1/\alpha}|\omega_1|,
\]
which is Eq.~\eqref{eq: proposition1 ordering}. Moreover, $\big(S(\omega)-G_t(\omega)\big)/S(\omega)=1/\big(1+\mathrm{SNR}_t(\omega)\big)$, so the maximum achievable MSE reduction at high-SNR (low) frequencies is pinned close to their energy $S(\omega)$, while that at low-SNR (high) frequencies is capped by their small energy. This proves the proposition.

\begin{remark}[Noise schedules and prediction targets]
The proof only uses $\pmb{x}_t=a_t\pmb{x}_0+\sigma_t\pmb{\varepsilon}$ with $\pmb{\varepsilon}\sim\mathcal{N}(\pmb{0},\pmb{I})$ independent of $\pmb{x}_0$. For a general schedule $(a_t,\sigma_t)$, e.g., the linear interpolant $a_t=1-t$, $\sigma_t=t$ used by SiT, all results hold with $\bar{\alpha}_t$ and $1-\bar{\alpha}_t$ replaced by $a_t^2$ and $\sigma_t^2$, i.e., $\mathrm{SNR}_t(\omega)=a_t^2S(\omega)/\sigma_t^2$. Moreover, the $\pmb{\varepsilon}$-prediction and velocity-prediction targets are affine in $\pmb{x}_0$ given $\pmb{x}_t$: $\pmb{\varepsilon}=(\pmb{x}_t-a_t\pmb{x}_0)/\sigma_t$ and $\pmb{v}=\dot a_t\pmb{x}_0+\dot\sigma_t\pmb{\varepsilon}=(\dot a_t-\dot\sigma_t a_t/\sigma_t)\,\pmb{x}_0+(\dot\sigma_t/\sigma_t)\,\pmb{x}_t$. For corresponding prediction functions, their MSEs equal the $\pmb{x}_0$-prediction MSE multiplied by the frequency-independent factors $(a_t/\sigma_t)^2$ and $(\dot a_t-\dot\sigma_t a_t/\sigma_t)^2$, respectively. These factors only reweight timesteps; at each $t$, comparisons of MSE reductions are preserved when the reference predictor is transformed consistently between parameterizations.
\end{remark}

\paragraph{Centered spectra and mean prediction.} To connect the second-moment spectrum in the proposition to variance-based measurements, let $\pmb{m}=\mathbb{E}[\pmb{x}_0]$, $\mu(\omega)=\mathcal{F}(\pmb{m})(\omega)$, and $V(\omega)=\mathrm{Var}[y_0(\omega)]=S(\omega)-|\mu(\omega)|^2$. Applying the same argument to $\pmb{x}_0-\pmb{m}$ and $\pmb{x}_t-a_t\pmb{m}$ leaves the optimal denoising error unchanged and gives
\[
V(\omega)\frac{\mathrm{SNR}^{\mathrm{c}}_t(\omega)}{1+\mathrm{SNR}^{\mathrm{c}}_t(\omega)}
\;\le\;\Delta^{\mathrm{c}}_t(\omega)\;\le\;V(\omega),
\qquad
\Delta^{\mathrm{c}}_t(\omega)\triangleq V(\omega)-\ell_{t,\omega}(f^\star),
\]
where $\mathrm{SNR}^{\mathrm{c}}_t(\omega)=a_t^2V(\omega)/\sigma_t^2$. Here $\Delta^{\mathrm{c}}_t$ measures the maximum MSE reduction relative to the mean predictor $\pmb{m}$. The same construction applies within each class by taking class-conditional expectations. This is the centered version of the Wiener bounds used in Appendix~\ref{append: predictability}.

\subsection{Measuring Frequency Predictability of a Pixel Diffusion Model}
\label{append: predictability}

\paragraph{Setup.} We use a JiT-H/16 pixel diffusion model \citep{li2026back} trained on ImageNet $256\times256$ for 600 epochs, which uses $\pmb{x}_t=(1-t)\,\pmb{x}_0+t\,\pmb{\varepsilon}$ and directly predicts $\hat{\pmb{x}}_0$. We take all 1300 training images of six classes (indices 0, 199, 399, 599, 799, and 999), 51 uniformly spaced timesteps in $[0,1]$, and 4 noise draws per image and timestep, conditioning on the ground-truth label without guidance.

\paragraph{Predictability.} All expectations and variances below are taken within each class. For each frequency $\omega$ of the unitary 2D DFT, we estimate $V(\omega)=\mathrm{Var}[y_0(\omega)]$, $P_t(\omega)=\mathrm{Var}[\hat y_0(\omega)]$, and $U_t(\omega)=\mathbb{E}|y_0(\omega)-\hat y_0(\omega)|^2$ over the samples, and define $R_t(\omega)=P_t/(P_t+U_t)$. For a Bayes-optimal denoiser, $P_t+U_t=V$ and $R_t=1-U_t/V=\Delta^{\mathrm{c}}_t/V$, the normalized maximum MSE reduction relative to the class-mean predictor, as derived in Appendix~\ref{append: Proof of Proposition}. For the evaluated model, the two expressions $P_t/(P_t+U_t)$ and $1-U_t/V$ differ by $0.002$ on average. We average $P_t$, $U_t$, and $V$ over color channels and integer-radius rings before computing $R_t$, and then average $R_t$ over the six classes. Figure~\ref{fig: predictability_heatmap} is restricted to the Nyquist radius $0.5$ cycles/pixel (the corners beyond it hold $0.2\%$ of the spectral variance); we omit $t<0.05$, which JiT rarely samples during training. In the displayed range, each class's measured $R_t(\omega)$ lies above the centered Wiener reference $\mathrm{SNR}^{\mathrm{c}}_t(\omega)/(1+\mathrm{SNR}^{\mathrm{c}}_t(\omega))$ at all grid points, with $\mathrm{SNR}^{\mathrm{c}}_t(\omega)=(1-t)^2V(\omega)/t^2$. The six classes share the same pattern (pairwise Pearson correlation of $R_t$ at least $0.94$).

\paragraph{Band shares of the loss reduction.} The MSE reduction relative to the class-mean predictor is $V(\omega)-U_t(\omega)$. For the reported band shares, we first average $V$ and $U_t$ over classes and take the positive part $\max\{V-U_t,0\}$. We weight this quantity over $t$ by the training objective of JiT, i.e., its timestep density times the velocity-loss weight $1/\max(t,0.05)^2$, and sum it over each frequency band, excluding DC. Frequencies below $0.1$ cycles/pixel ($3.2\%$ of the coefficients, $90.9\%$ of the spectral variance) account for $94.8\%$ of the positive reduction, while frequencies above $0.25$ cycles/pixel ($80.4\%$ of the coefficients) account for only $1.1\%$. Uniform weighting over $t$ gives nearly the same result ($94.4\%$ below $0.1$ cycles/pixel).

\subsection{An Information-Theoretic Motivation for Spectral Flattening}
\label{append: ESM from an Information Theory Perspective}

Natural images commonly have a power-law spatial spectrum,
$S_x(\omega)\propto \|\omega\|^{-\alpha}$ with $\alpha>0$, implying that low frequencies carry substantially more energy. We use a maximum-entropy reference model to motivate reducing this concentration in the latent spectrum. Specifically, under a fixed power budget, the entropy-maximizing Gaussian spectrum is flat. This reference model motivates the flattening component of the ESM hypothesis; retaining a power-law form is a separate modeling choice motivated by natural-image statistics.

\begin{remark}[\textbf{Gaussian reference for a maximum-entropy upper bound}]
We do not assume that the latent field is exactly Gaussian.
Instead, we use a Gaussian reference model for second-order spectral analysis: among all stationary random fields sharing the same PSD (equivalently, the same second-order statistics), the Gaussian one has the maximum differential entropy rate.
Therefore, analyzing the Gaussian case provides a valid maximum-entropy upper-bound argument under a fixed PSD (power) constraint.
\end{remark}

\begin{lemma}[\textbf{Maximum-entropy spectrum under a finite power budget implies flattening effect}]
Let $z$ be a zero-mean, wide-sense stationary latent random field with power spectral density $S_z(\omega)>0$ defined on a continuous frequency domain $\Omega$.
Assume the latent is subject to an energy budget
\begin{equation}
\int_{\Omega} S_z(\omega)\,d\omega = P,
\label{eq:latent_power_budget}
\end{equation}
for some constant $P>0$.
Under the Gaussian reference model, the differential entropy rate is maximized by a flat spectrum:
\begin{equation}
S_z^\star(\omega)=\frac{P}{\mathrm{Vol}(\Omega)}
\quad \text{for a.e. } \omega\in\Omega,
\label{eq:flat_optimal_psd}
\end{equation}
where $\mathrm{Vol}(\Omega)$ denotes the volume of the frequency domain.
Hence, the entropy-maximizing latent distribution exhibits a whitening (spectral-flattening) tendency.
\label{lemma:1}
\end{lemma}

\begin{proof}[Sketch of Proof]
For a zero-mean stationary Gaussian random field, the differential entropy rate admits the spectral representation
\begin{equation}
h(z)=\frac12\int_{\Omega}\log S_z(\omega)\,d\omega + \mathrm{const}.
\label{eq:entropy_rate_psd}
\end{equation}
Therefore, maximizing $h(z)$ under Eq.~\eqref{eq:latent_power_budget} is equivalent to maximizing
$\int_{\Omega}\log S_z(\omega)\,d\omega$ under the same constraint.

Since $\log(\cdot)$ is strictly concave on $\mathbb{R}_{>0}$, applying Jensen's inequality yields

\begin{equation}
\frac{1}{\mathrm{Vol}(\Omega)}\int_{\Omega}\log S_z(\omega)\,d\omega
\le
\log\!\left(\frac{1}{\mathrm{Vol}(\Omega)}\int_{\Omega} S_z(\omega)\,d\omega\right)
=
\log\!\left(\frac{P}{\mathrm{Vol}(\Omega)}\right),
\label{eq:jensen_flat}
\end{equation}
where the last equality uses the energy budget Eq.~\eqref{eq:latent_power_budget}.
Moreover, equality holds if and only if $S_z(\omega)$ is constant almost everywhere on $\Omega$.
Hence the maximizer is
\[
S_z^\star(\omega)=\frac{P}{\mathrm{Vol}(\Omega)}
\quad \text{for a.e. } \omega\in\Omega.\qedhere
\]
\end{proof}

\noindent\textbf{Implications for the ESM hypothesis.} Lemma~\ref{lemma:1} identifies a flat spectrum as the maximum-entropy endpoint under a fixed power budget. This reference model motivates spectral flattening as a way to distribute energy more evenly across frequencies. The power-law structure of natural images provides a separate motivation for retaining a nonuniform spectral organization. Combining these two considerations, we model a diffusion-friendly latent spectrum as a partially flattened power law,
\[
S_z(\omega)\propto \|\omega\|^{-(\alpha-\delta)}\quad\text{with}\quad \delta>0,
\]
where the retained power-law form is a modeling assumption, and the degree of flattening is determined empirically. The lemma supplies a motivation for the direction of spectral change rather than an optimality guarantee for the proposed latent spectrum.

\subsection{Spectrum Analysis of UAE and VA-VAE}
\label{append: Analysis of Low Frequency Alignment in UAE}

Recall that our ESM hypothesis states that the latent spectrum of $\pmb{z} = E(\pmb{x})$ should follow an approximately power-law PSD $S_{\pmb{z}}(\omega)\propto |\omega|^{-(\alpha-\delta)}$ with $\delta>0$ flattening the natural-image spectrum $S_{\pmb{x}}(\omega)\propto |\omega|^{-\alpha}$. We now analyze the methods UAE and VA-VAE which align the latent $z$ with DINOv2 features \citep{oquab2024dinov2}. We first measure the spectrum distribution of DINOv2 on the ImageNet $256\times256$ dataset. The results in Figure~\ref{fig: dinov2_spectrum} demonstrate that the spectrum of DINOv2 features approximately follows a power-law PSD $S_{\pmb{z}}(\omega)\propto |\omega|^{-(\alpha-\delta)}$ with $\delta=1.0$. In detail, the RGB spectrum distribution is evaluated using 50,000 random images. Then, the power-law target ($\delta=1.0$) is formed by flattening the RGB spectrum.

\begin{figure}[htb]
\vskip -0.1in
\begin{center}
\centerline{\includegraphics[width=0.7\columnwidth]{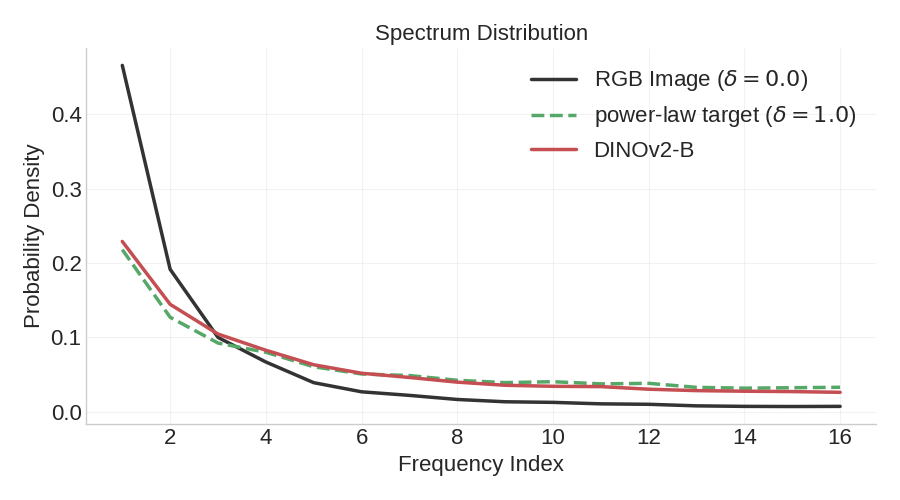}}
\caption{Spectrum distributions of the input image and DINOv2 features on ImageNet $256\times256$.
}
\label{fig: dinov2_spectrum}
\end{center}
\vskip -0.0in
\end{figure}

In addition, we measure the practical spectrum distribution of VA-VAE latents on ImageNet $256\times256$. Figure~\ref{fig: vavae_spectrum} shows that the VA-VAE latent spectrum is close to the power-law target PSD ($\delta=1.0$). This spectral similarity supports a connection between DINOv2 feature alignment and ESM, although feature alignment imposes additional constraints beyond PSD matching.

\begin{figure}[htb]
\vskip -0.1in
\begin{center}
\centerline{\includegraphics[width=0.7\columnwidth]{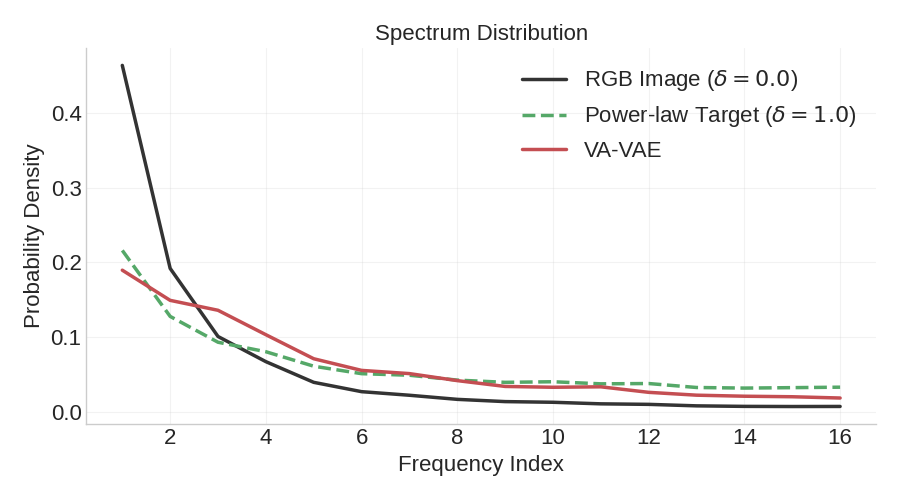}}
\caption{Spectrum distributions of the input image and VA-VAE latents on ImageNet $256\times256$.
}
\label{fig: vavae_spectrum}
\end{center}
\vskip -0.0in
\end{figure}

\subsection{Spectral Connections to Scale Equivariance and EQ-VAE}
\label{append: Spectral Connections to Scale Equivariance and EQ-VAE}

Scale Equivariance \citep{skorokhodov2025improving} downsamples the latent $\pmb{z}$ by 2$\times$ or 4$\times$ and trains the decoder to reconstruct the correspondingly downsampled image $\pmb{x}$. EQ-VAE \citep{kouzelis2025eq} uses both scaling and rotation to regularize the decoder. Our spectral interpretation concerns its scaling component; the rotation-equivariance constraint is a separate aspect of the method.

Downsampling with anti-aliasing combines low-pass filtering with resampling. A concrete DCT-domain connection is provided by \citet{ning2025dctdiff}: for a $2B\times 2B$ image block and its $B\times B$ average-pooled version, the DCT coefficients of the latter can be approximated by the top-left $B\times B$ coefficients of the former, with frequency-dependent rescaling. These rescaled coefficients can then be used to approximately reconstruct the downsampled image, supporting a spectral interpretation of downsampling through low-frequency selection. DSM generalizes this filtering principle by applying shared masks to latents and images, allowing more flexible low-pass patterns while enforcing frequency-aligned reconstruction. Figure~\ref{fig: DCT_upsampling} illustrates the rectangular low-frequency region associated with this DCT-domain approximation (left) and the triangular masks used by DSM (right).

\begin{figure}[ht]
\centering
\includegraphics[width=0.65\columnwidth]{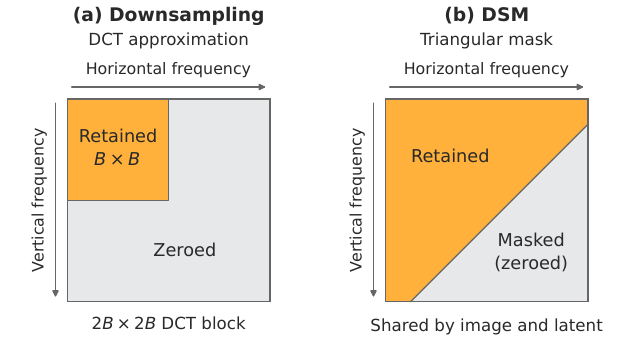}
\caption{DCT-domain comparison of downsampling and DSM. Left: the rectangular low-frequency region used to approximate $2\times$ average-pooling, with frequency-dependent rescaling omitted. Right: an example of DSM's triangular high-frequency mask. Orange regions are retained; gray regions are zeroed out.}
\label{fig: DCT_upsampling}
\end{figure}

\subsection{Proof of Proposition~\ref{proposition:2}}
\label{append: Proof of Proposition 2}

\begin{proof}
By definition of $\bar u$ and expanding the square,
\[
\frac{1}{T}\sum_{t=1}^T \|u_t-\bar u\|_2^2
=
\frac{1}{T}\sum_{t=1}^T \|u_t\|_2^2
-\|\bar u\|_2^2.
\]
Since $\|u_t\|_2=1$ for all $t$, the first term equals $1$, hence
\begin{equation}
\label{eq:rmsc_mean_dir_appendix}
\mathrm{RMSC}(x)^2 = 1 - \|\bar u\|_2^2.
\end{equation}

Let $\mathcal{C}\in\mathbb{R}^{T\times T}$ be an orthonormal DCT matrix ($\mathcal{C}^\top\mathcal{C}=I_T$).
Applying it along $t$ to each feature dimension gives coefficient vectors $U_k\in\mathbb{R}^D$.
By Parseval's identity for an orthonormal transform,
\[
\sum_{t=1}^T \|u_t\|_2^2 = \sum_{k=0}^{T-1}\|U_k\|_2^2.
\]
Moreover, the DC coefficient equals the mean up to $\sqrt{T}$:
\[
U_0
=
\sum_{t=1}^T \mathcal{C}_{0,t}\,u_t
=
\frac{1}{\sqrt{T}}\sum_{t=1}^T u_t
=
\sqrt{T}\,\bar u,
\]
so $\|U_0\|_2^2=T\|\bar u\|_2^2$. Since $\sum_{t=1}^T\|u_t\|_2^2=T$, we obtain
\[
\sum_{k=1}^{T-1}\|U_k\|_2^2
=
T-\|U_0\|_2^2
=
T-T\|\bar u\|_2^2
=
T\,\mathrm{RMSC}(x)^2,
\]
where the last step uses Eq.~\eqref{eq:rmsc_mean_dir_appendix}. Dividing by $T$ completes the proof.
\end{proof}

\subsection{Training Parameters}
\label{append: Training Parameters}

We present the complete training parameters of Autoencoders in Table~\ref{tab: training parameters}. All models (SD-VAE, ESM-AE, DSM-AE) are trained on CelebA $256\times256$ using 4 A100 GPUs and on ImageNet $256\times256$ using 8 A100 GPUs.

In addition, Table~\ref{tab: REPA training parameters} shows the training parameters of REPA, iREPA, and REPA-DoG on ImageNet $256\times256$. Experiments are implemented using 4 A100 GPUs.

\begin{table}[htb]
\caption{
Training hyperparameters of autoencoders on CelebA $256\times256$ and ImageNet $256\times256$.}
\label{tab: training parameters}
\centering
\small
\begin{tabular}{@{}lclll|cl@{}}
\toprule
Dataset & \multicolumn{4}{c|}{CelebA 256 (f8d4, f16d16)} & \multicolumn{2}{l}{ImageNet 256 (f16d16)} \\ \midrule
Model & SD-VAE & Scale Equivariance & ESM-AE & DSM-AE & \multicolumn{1}{l}{SD-VAE} & DSM-AE \\ \midrule
network parameters & \multicolumn{4}{c|}{83.7M (f8d4), 64.4M (f16d16)} & \multicolumn{2}{c}{64.4M} \\
VAE learning rate & \multicolumn{4}{c|}{$5\times10^{-5}$} & \multicolumn{2}{c}{$5\times10^{-5}$} \\
GAN learning rate & \multicolumn{4}{c|}{$5\times10^{-5}$} & \multicolumn{2}{c}{$5\times10^{-5}$} \\
GAN start step & \multicolumn{4}{c|}{50k} & \multicolumn{2}{c}{50k} \\
learning rate schedule & \multicolumn{4}{c|}{constant with warmup} & \multicolumn{2}{c}{constant with warmup} \\
batch size & \multicolumn{4}{c|}{48} & \multicolumn{2}{c}{128} \\
training steps & \multicolumn{4}{c|}{500k} & \multicolumn{2}{c}{600k} \\
optimizer & \multicolumn{4}{c|}{AdamW} & \multicolumn{2}{c}{AdamW} \\
weight decay & \multicolumn{4}{c|}{0.005} & \multicolumn{2}{c}{0.005} \\
mixed precision & \multicolumn{4}{c|}{bf16} & \multicolumn{2}{c}{bf16} \\
$\lambda_1$ (LPIPS loss) & \multicolumn{4}{c|}{0.5} & \multicolumn{2}{c}{0.5} \\
$\lambda_2$ (GAN loss) & \multicolumn{4}{c|}{0.5} & \multicolumn{2}{c}{0.5} \\
$\delta$ (ESM) & \multicolumn{1}{l}{-} & - & 1.0 & - & \multicolumn{1}{l}{-} & - \\
$\beta$ (ESM loss) & \multicolumn{1}{l}{-} & - & 0.01 & - & \multicolumn{1}{l}{-} & - \\
KL loss & \multicolumn{1}{l}{0} & 0 & 0 & 0 & \multicolumn{1}{l}{0} & 0 \\ \bottomrule
\end{tabular}
\end{table}

\begin{table}[H]
\caption{
Training hyperparameters of REPA, iREPA, and REPA-DoG on ImageNet $256\times256$.}
\label{tab: REPA training parameters}
\centering
\small
\begin{tabular}{@{}lcll@{}}
\toprule
Model & \multicolumn{1}{l}{REPA} & iREPA & REPA-DoG \\ \midrule
diffusion network & \multicolumn{3}{c}{SiT-B/2} \\
foundation model & \multicolumn{3}{c}{DINOv2-B} \\
encoder depth & \multicolumn{3}{c}{4} \\
alignment coefficient & \multicolumn{3}{c}{0.5} \\
learning rate & \multicolumn{3}{c}{$10^{-4}$} \\
batch size & \multicolumn{3}{c}{256} \\
training steps & \multicolumn{3}{c}{400k} \\
optimizer & \multicolumn{3}{c}{Adam} \\
mixed precision & \multicolumn{3}{c}{fp16} \\
FID samples & \multicolumn{3}{c}{50,000} \\
diffusion sampling steps & \multicolumn{3}{c}{100} \\
diffusion sampler & \multicolumn{3}{c}{Euler--Maruyama} \\
CFG scale & \multicolumn{3}{c}{1.8} \\ \midrule
projection layer & \multicolumn{1}{l}{MLP} & \multicolumn{2}{c}{Conv} \\
spatial normalization & - & \checkmark & - \\
DoG & - & - & \checkmark \\
normalization std & - & 0.6 & 1.0 \\ \bottomrule
\end{tabular}
\end{table}

\subsection{Design of Frequency Mask \texorpdfstring{$\mathcal{M}$}{M}}
\label{append: Design of Frequency Mask}

Considering that the frequency in the DCT space follows a zigzag order (see Figure~\ref{fig: mask_design}), where low frequencies lie in the top-left corner and high frequencies are in the bottom right, we design the frequency mask set $\mathcal{M}$ as a series of triangular shapes to perform progressive high-frequency filtering. In detail, we follow the convention of the JPEG codec \citep{wallace1991jpeg} and use the $8\times8$ DCT block in practice. Thus, the number of diagonal rows $n$ of the mask can uniquely define a specific mask $M\sim \mathcal{M}$. Figure~\ref{fig: mask_design} further shows some concrete samples when $n=4$, $n=8$ and $n=12$.

\begin{figure}[htb]
\vskip -0.1in
\begin{center}
\centerline{\includegraphics[width=0.9\columnwidth]{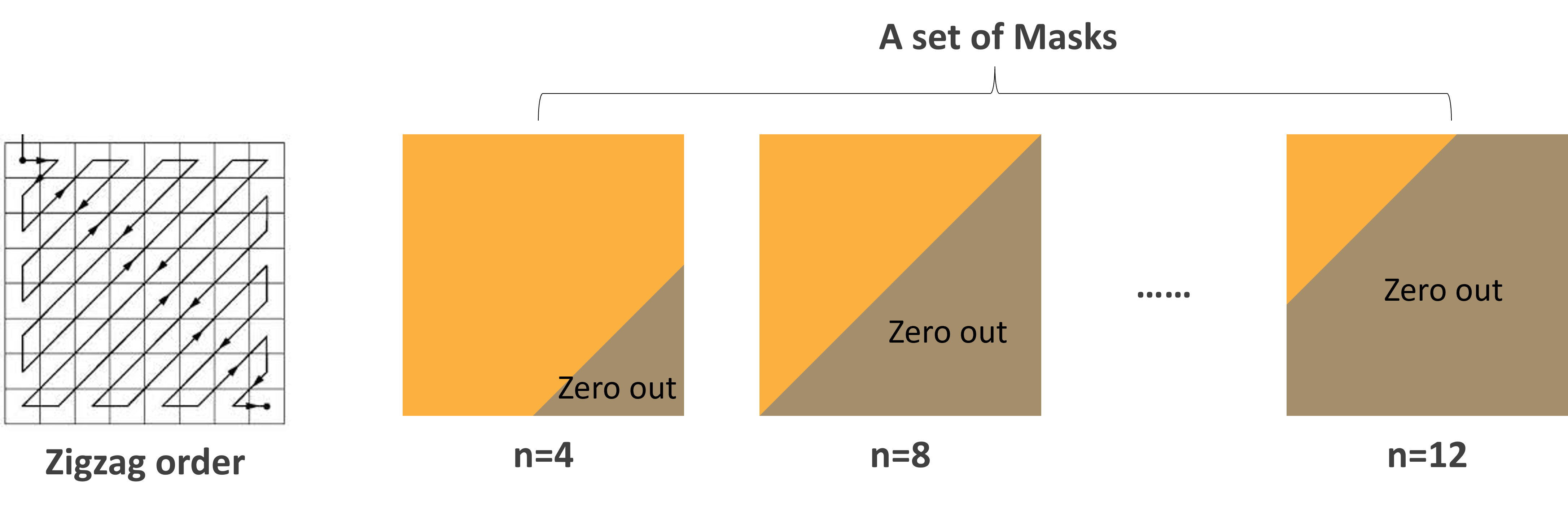}}
\caption{The frequency in the DCT space follows a zigzag order where low frequencies lie in the top-left corner, and high frequencies are in the bottom right. Therefore, we design the frequency mask as a set of triangular shapes to perform different levels of high-frequency filtering. On the right side, we show some concrete samples of the mask $M$ using $n=4$, $n=8$ and $n=12$ in the $8\times8$ DCT space.
}
\label{fig: mask_design}
\end{center}
\vskip -0.1in
\end{figure}

\clearpage
\subsection{Ablation Studies of ESM and DSM}
\label{append: Ablation Studies of ESM and DSM}

\paragraph{Ablation study of ESM.}
There are two parameters in the ESM regularization: the PSD flattening factor $\delta$ and the ESM loss weight $\beta$. We perform the ablation study in terms of $\delta$ and $\beta$. Experimental results in Table~\ref{tab: ablation_of_ESM} indicate that $\delta=1.0$ and $\beta=0.01$ yield the best diffusability on CelebA $256\times256$ dataset. The optimal $\delta=1.0$ is consistent with the DINOv2 feature spectrum we have analyzed in Appendix~\ref{append: Analysis of Low Frequency Alignment in UAE}.

\begin{table}[H]
\caption{
Ablation study of ESM-AE on CelebA $256\times256$. Diffusion steps denote the number of training iterations at convergence.}
\label{tab: ablation_of_ESM}
\centering
\small
\begin{tabular}{@{}llllllll@{}}
\toprule
\multicolumn{1}{c}{\multirow{2}{*}{Models}} & \multirow{2}{*}{$\delta$} & \multirow{2}{*}{$\beta$} & \multicolumn{5}{c}{f16d16} \\ \cmidrule(l){4-8} 
\multicolumn{1}{c}{} &  &  & rFID $\downarrow$ & PSNR $\uparrow$ & SSIM $\uparrow$ & steps & gFID $\downarrow$ \\ \midrule
ESM-AE & 0.6 & 0.01 & 1.56 & 30.52 & 0.767 & 100k & 5.90 \\
ESM-AE & 1.0 & 0.01 & 1.56 & 31.30 & 0.782 & 100k & \textbf{5.84} \\
ESM-AE & 1.4 & 0.01 & 1.61 & 31.62 & 0.790 & 100k & 6.57 \\ \midrule
ESM-AE & 1.0 & 0.1 & 1.63 & 30.86 & 0.777 & 100k & 6.82 \\
ESM-AE & 1.0 & 0.01 & 1.56 & 31.30 & 0.782 & 100k & \textbf{5.84} \\
ESM-AE & 1.0 & 0.001 & 1.82 & 31.19 & 0.781 & 100k & 6.00 \\ \bottomrule
\end{tabular}
\end{table}

\paragraph{Ablation study of DSM.}
There is only one hyperparameter in the DSM regularization: the design of the DCT frequency mask set $\mathcal{M}$. We test various designs of $\mathcal{M}$, in which $n=\{8, 12\}$ means $\mathcal{M}$ consists of two mask samples and they remove $8$ and $12$ diagonal rows of frequencies, respectively. Moreover, we maintain the mask sample $n=0$ in $\mathcal{M}$ to train DSM-AE on the original input image. The results in Table~\ref{tab: ablation_of_DSM} reveal that a proper element density in the mask set $\mathcal{M}$ leads to the optimal diffusability on CelebA $256\times256$. This phenomenon is consistent with our expectation: too few elements in the mask set $\mathcal{M}$ lead to a weak DSM constraint, while too many mask elements in $\mathcal{M}$ would reduce the chance of training the original image ($n=0$).

\begin{table}[H]
\caption{
Ablation study of DSM-AE on CelebA $256\times256$. Diffusion steps denote the number of training iterations at convergence.}
\label{tab: ablation_of_DSM}
\centering
\small
\begin{tabular}{@{}lllllll@{}}
\toprule
\multicolumn{1}{c}{\multirow{2}{*}{Models}} & \multirow{2}{*}{Mask set $\mathcal{M}$} & \multicolumn{5}{c}{f16d16} \\ \cmidrule(l){3-7} 
\multicolumn{1}{c}{} &  & rFID $\downarrow$ & PSNR $\uparrow$ & SSIM $\uparrow$ & steps & gFID $\downarrow$ \\ \midrule
DSM-AE & $n=\{0, 8, 12\}$ & 1.40 & 29.64 & 0.749 & 100k & 5.02 \\
DSM-AE & $n=\{0, 8, 10, 12\}$ & 1.39 & 29.32 & 0.745 & 100k & \textbf{4.49} \\
DSM-AE & $n=\{0, 6, 8, 10, 12\}$ & 1.38 & 29.52 & 0.746 & 100k & 4.61 \\
DSM-AE & $n=\{0, 8, 10, 12, 13\}$ & 1.23 & 28.71 & 0.728 & 100k & 5.33 \\ \bottomrule
\end{tabular}
\end{table}

\end{document}